\documentclass[11pt]{article}
\usepackage{hyperref}

\usepackage[toc,page]{appendix}
\usepackage[utf8]{inputenc}
\usepackage{amsmath, latexsym, stmaryrd, amsthm, caption,amssymb
	,xcolor,subfig,amssymb,bbm, multirow, multicol, makecell,graphics,pgfplots,bm,
	,tabularx,a4wide,xurl,breakcites,enumitem,mathtools
}
\usepackage[nameinlink]{cleveref}
\hypersetup{
	colorlinks=true,
	linkcolor=blue,         
	citecolor=blue,         
	urlcolor=blue,
	pdfstartview=FitV
}

\numberwithin{equation}{section}
\usepackage[T1]{fontenc}
\usepackage[normalem]{ulem}
\newcolumntype{Y}{>{\raggedleft\arraybackslash}X}

\newcommand{\N}{\mathbb N}
\newcommand{\R}{\mathbb R}

\newcommand{\E}{\mathbb E}

\renewcommand{\S}{\mathbb S}

\renewcommand{\P}{\mathbb P}

\newcommand{\Var}{\mathrm{Var}}

\newcommand{\quadre}[1]{\left[ #1 \right]}
\newcommand{\tonde}[1]{\left( #1 \right)}
\newcommand{\abs}[1]{\left| #1 \right|}
\newcommand{\di}{\mathrm d}

\def\be{\begin{eqnarray}}
	\def\ee{\end{eqnarray}}
\def\ben{\begin{eqnarray*}}
	\def\een{\end{eqnarray*}}

\DeclareMathOperator{\cov}{Cov}
\newtheorem{theorem}{Theorem}[section] 
\newtheorem{lemma}[theorem]{Lemma}     
\newtheorem{proposition}[theorem]{Proposition}

\theoremstyle{definition}
\newtheorem{definition}[theorem]{Definition}

\newtheorem{remark}[theorem]{Remark}

\pgfplotsset{compat=1.18}
\usepackage{subfig}

\usepackage{libertineRoman}
\usepackage{authblk}
	\title{Spectral Complexity of Deep Neural Networks}

\author{Simmaco Di Lillo}
\author{Domenico Marinucci}
\author{Michele Salvi}
\author{Stefano Vigogna}
\affil{RoMaDS - Department of Mathematics, University of Rome Tor Vergata, 	Rome, Italy \newline
{\tt{\{dilillo, marinucc, salvi, vigogna\}@mat.uniroma2.it}}}

\date{}                     

\begin{document}

	\maketitle

	\begin{abstract}
		It is well-known that randomly initialized, push-forward, fully-connected neural networks weakly converge to isotropic Gaussian processes, in the limit where the width of all layers goes to infinity. In this paper, we propose to use the angular power spectrum of the limiting fields to characterize the complexity of the network architecture. In particular, we define sequences of random variables associated with the angular power spectrum, and provide a full characterization of the network complexity in terms of the asymptotic distribution of these sequences as the depth diverges. On this basis, we classify neural networks as low-disorder, sparse, or high-disorder; we show how this classification highlights a number of distinct features for standard activation functions, and in particular, sparsity properties of ReLU networks. Our theoretical results are also validated by numerical simulations.
	\end{abstract}
		\noindent \textbf{MSCcodes}:  	68T07, 60G60, 33C55, 62M15.

	\section{Introduction}
	Allowing depth in neural networks has been instrumental in elevating them to the forefront of machine learning,
	paving the way for unprecedented results spanning from image recognition to natural language processing,
	up to the latest wonders of generative AI.
	Mirroring human abstraction, deep artificial networks extract relevant features in a hierarchical fashion, constructing complex representations from simpler ones,
	and discerning high-dimensional patterns in reduced dimensions.
	On the other hand, just stacking more layers comes with its own risks, including overfitting, vanishing gradients, and general optimization instability,
	which makes depth a resource to use with care.
	Identifying measures of complexity capturing benefits and drawbacks of depth
	can thus provide a principled way to explain empirical behaviors and guide algorithmical choices.

	Besides intuitions and practical evidence, a theoretical understanding of the role of depth in neural architectures remains a challenging endeavor.
	From an approximation perspective,
	the problem has been addressed in terms of \emph{depth separation}.
	This approach goes beyond classical universal theorems \cite{cybenko,hornik,leshno,pinkus},
	quantifying how the approximation error can scale exponentially along the width
	and polynomially along the depth
	\cite{pmlr-v49-telgarsky16,pmlr-v49-eldan16,mhaskar2016deep, pmlr-v65-daniely17a,pmlr-v70-safran17a,MR4134774,JMLR:v23:21-1109}.
	Recent work has extended these results to the \emph{infinite-width limit},
	shifting the representation cost from the number of units to a minimal weight norm \cite{parkinson2024depth}.
	Other approaches have explored depth's impact by estimating functional properties
	such as slope changes and number of linear regions in ReLU networks \cite{NIPS2014_109d2dd3,pmlr-v97-hanin19a,GOUJON2024115667}.
	Such quantities can be convexified by specific families of seminorms,
	thereby enforcing practical regularization strategies \cite{doi:10.1137/22M147517X}.
	More generally, people have studied how the topology of the network changes with respect to its architecture,
	tracking the dependence of topological invariants on depth and activation function \cite{6697897,10.5555/3016100.3016187}.
	Within this framework, Betti numbers have been notably considered
	and used to define suitable notions of complexity.
	More traditional model complexities such as VC-dimension and Rademacher
	have also been bounded for neural networks \cite{NIPS1998_bc731692,JMLR:v20:17-612,NIPS2017_b22b257a}.
	However, such bounds can be too loose in typical overparameterized regimes.

	An alternative point of view that is particularly relevant for our study
	is provided by the theories of reproducing kernels \cite{aronszajn50reproducing}
	and Gaussian processes \cite{10.5555/1162254}.
	The link with neural networks may be sketched as follows.
	In suitable infinite-width limits, sometimes called \emph{kernel regimes},
	neural networks are equivalent to Gaussian processes,
	and are therefore characterized by positive definite kernels
	\cite{neal2012bayesian,NIPS1996_ae5e3ce4,NIPS2016_abea47ba,lee2018deep,g.2018gaussian,10.1214/23-AAP1933,cammarota_marinucci_salvi_vigogna_2023,balas24,favaro2023quantitative,NEURIPS2018_5a4be1fa}.
	Furthermore, the kernel corresponding to a deep neural network has an interesting compositional structure,
	namely it can be obtained by iteration of a fixed (shallow) kernel as many times as the number of hidden layers.
	This observation has sparked the idea of giving depth to kernel methods,
	with the hope of enhancing, if not feature learning, at least the expressivity of the model
	\cite{NIPS2009_5751ec3e,pmlr-v51-wilson16,JMLR:v20:17-621,JMLR:v24:23-0538}.
	In this context, potential benefits may be proved by checking
	whether the associated reproducing kernel Hilbert space (RKHS) actually expands when iterating the kernel \cite{JMLR:v24:23-0538}.
	In a somewhat opposite yet related spirit,
	one can use the infinite-width limits to study RKHS as approximations of neural network models.
	It turns out that, even when the kernel is deep, the corresponding RKHS can be shallow,
	that is, it can be constant with respect to the depth of the kernel.
	More precisely, \cite{bietti2021deep} have shown that the spectrum of the integral operator
	associated to a deep ReLU kernel has the same asymptotic order regardless of the number of layers.
	As a consequence, ReLU RKHS of any depth are all equivalent,
	and therefore their structure fails in capturing the additional complexity induced by the depth. In this paper we argue that the failure of the RKHS structure does not imply the general unsuitability of kernel regimes for studying depth in neural architectures.
	While \cite{bietti2021deep} look at the tail of the spectrum, which indeed completely determines the RKHS, we look at the \emph{whole} spectrum, which may well depend on the depth even when the tail does not.

	\subsection*{Our approach} In this paper we  study the spectrum of the  random feature kernel associated to infinitely wide neural networks as a function of depth and activation function.
	Regarding neural networks as random fields on the sphere, such a spectrum coincides with the \emph{angular power spectrum} of the field,
	i.e.~the variance of the coefficients of its harmonic expansion.
	Normalizing the variance of the neural network, we identify the angular power spectrum with a probability distribution on the non-negative integers,
	that we call the \emph{spectral law} of the neural network.
	We denote by $X_L$ the associated random variable,
	stressing its dependence on the depth $L$.
	Our first main result (\Cref{main_theorem_1}) characterizes neural networks
	based on the behavior of the moments of $X_L$.
	Depending on the form of the activation function,
	we distinguish three regimes:
	a \emph{low-disorder} case (including the Gaussian activation), where the finite moments of $X_L$ decay exponentially to zero as $L \to \infty$;
	a \emph{sparse case} (including ReLU and Leaky ReLU), where the first moment goes to $0$, the second is uniformly bounded and the others grow polynomially -- when they exist;
	and a \emph{high-disorder} case (including the hyperbolic tangent), where the finite moments diverge exponentially.
	Our second main result (\Cref{main_theorem_3})
	studies the asymptotics of the associated random fields $T_L$ and its derivatives as $L\to\infty$:
	in the low-disorder case, $T_L$ is asymptotically constant
	and the derivatives converge to $0$ exponentially fast;
	in the sparse case, $T_L$ is asymptotically constant but its derivatives do not vanish, which can be interpreted as a random field living on low frequencies but with some isolated spikes;
	in the high-disorder case, $T_L$ is not asymptotically constant and its derivatives diverge exponentially, so that the field becomes more and more chaotic as $L$ grows.
	In particular, ReLU networks can be well approximated by low-degree polynomials in $L^2$, but not in Sobolev norms.
	This fact agrees with the intuition that ReLU induces sparsity/self-regularization \cite{pmlr-v15-glorot11a}.

	Finally, we introduce two indexes of complexity:
	the \emph{spectral effective support},
	which tells which multipoles capture
	most of the norm/variance of the network;
	and the \emph{spectral effective dimension},
	which measures the total dimension of the corresponding eigenspaces.
	Consistently with the previous considerations, these quantities are surprisingly low for ReLU networks: we show numerically that 99\% or more of the norm is supported on less than a handful of spectral multipoles for arbitrarily large depth.

	In conclusion, our key contributions may be summarized as follows.
	\begin{enumerate}[topsep=4pt,itemsep=2pt,parsep=0pt]
		\item We propose a new framework for studying the role of depth in neural architectures.
		Our approach is based on the theory of random fields,
		and more specifically on the spectral decomposition of isotropic random fields on the hypersphere.
		\item Following this approach, we classify networks in three regimes
		where depth plays a significantly different role. In short, these regimes correspond to degenerate asymptotics (convergence to a trivial limit), asymptotic boundedness, and exponential divergence.
		\item In particular, we show that ReLU networks fall into the intermediate regime, characterized by convergence in $L^2$ and divergence in Sobolev norms. This suggests that ReLU networks have a sparse/self-regularizing property, which may allow them to go deeper with less risk of overfitting.
		\item Based on the above, we introduce a new simple notion of complexity for neural architectures,
		aimed at describing the effects of depth with respect to different choices of the activation function.
	\end{enumerate}

	\subsection*{Comparison with the existing literature}
	Studying neural networks through spectral properties is not new in the literature \cite{bietti2021deep,bietti2021approximation,xiao2022eigenspace,cagnetta2023can}.
	While previous works focus on the tail of the spectrum, we characterize its full distribution. Since the depth of the network affects the distribution, but may not change its tail, our approach is better suited to describe the behavior of the network as the depth increases.
	For instance, as argued earlier, neural networks of varying depth may live in the same Sobolev space as functions \cite{bietti2021deep},
	and yet reveal distinct behaviors as random fields. By studying the whole spectrum, one has access to the frequencies that dominate the oscillations of the field as the network  deepens (\Cref{main_theorem_1}), or can quantify the rate of explosion (or contraction) of its higher order derivatives (\Cref{main_theorem_3}).

	Our work is also related to the so-called \emph{edge of chaos} \cite{NIPS2004_f8da71e5,pmlr-v9-glorot10a,NIPS2016_14851003,schoenholz2017deep,pmlr-v97-hayou19a},
	which consists of a set of initializations
	that lie
	between an ordered and a chaotic phase.
	Initializing at this edge allows input information to well propagate through the layers of a deep network during gradient descent,
	proving the criticality of a good initialization for successful training.
	Under the hypothesis of unitary variance at every layer, our sparse regime corresponds to the edge of chaos.
	However, our approach yields stronger results. For example, in \cite{pmlr-v97-hayou19a} the authors show that for initializations both on the edge of chaos and in the ordered phase (corresponding to our low disorder regime) the covariance function converges to $1$. While this is consistent with our finding that $T_L$ goes to a constant in $L^2$, we are also able to control the convergence of the derivatives of the field (and hence the convergence of the field in Sobolev norms), marking a clear separation between sparse and low disorder phase.

	Our approach also yields other important advantages  that cannot be obtained by just controlling the covariance decay, as in the edge of chaos literature, or the tails of the spectrum. Indeed, the control of the full spectrum allows us to introduce the notion of spectral support, which gives an indication of the effective degree of liberty of the network (even in the case of networks with finite width, without Gaussian approximation). For example, we show that ReLU networks at inizialitation are well approximated in the $L^2$ sense by polynomials of low degree, whereas approximating them in high Sobolev norms requires  polynomials of higher and higher degree. Finally, combining the control of the spectral support with the study of higher order derivatives of the field, one remarkably obtains  mathematical evidence for the existence of spikes in the sparse regime (see  \Cref{spikes}), validating what has been observed  empirically for ReLU networks.

	\section{Background and notation}\label{sec::Background}

	In this section we recall some basics on isotropic random fields
	and their decomposition in spherical harmonics,
	as well as some known facts
	about random neural networks
	and their Gaussian process limit
	at infinite width.

	\subsection{Isotropic random fields on the sphere}
	Let $T:\, \S^d \times \Omega \to \R$ be a measurable application for some probability space $(\Omega, \mathcal F, \P)$. If $T$ is isotropic, meaning that the law of $T(\cdot)$ is the same as $T(g\cdot)$ for all $g\in \operatorname{SO}(d)$ (the special group of rotations in $\R^{d+1}$), and it has finite covariance, then the spectral representation
	\begin{equation}\label{espansione} T(x,\omega) = \sum_{\ell=0}^{\infty } \sum_{m=1}^{n_{\ell,d}} a_{\ell m}(\omega) Y_{\ell m }(x), \qquad x\in \S^d, \ \omega\in \Omega
	\end{equation}
	holds in $L^2(\Omega \times \mathbb{S}^d)$, i.e.
	$$ \lim_{M\to \infty} \int_\Omega \int_{\S^d} \abs{T(x,\omega) -\sum_{\ell=0}^{M} \sum_{m=1}^{n_{\ell,d}} a_{\ell m}(\omega) Y_{\ell m}(x)}^2 \di x \, \di {\mathbb P(\omega)} =0\ .$$
	Here we adopt a standard notation, and in particular we write $\{Y_{\ell m}\}$ for a $L^2$-orthonormal basis of real-valued spherical harmonics, which satisfy
	\begin{equation}\label{laplacian}\Delta_{\S^d} Y_{\ell m} = -\ell (\ell + d- 1) Y_{\ell m} \ , \ \ell\in \N
	\end{equation}
	where $\Delta_{\S^d}$ is the Laplace-Beltrami operator on the sphere. For all $\ell \in \N$, the dimension of the eigenspaces, i.e. the cardinality of the basis elements $\{Y_{\ell m}\}$,  is given by
	\begin{equation}\label{dim_spazio}
		\begin{aligned}
			&n_{0, d} = 1 \\
			&n_{\ell,d} = \frac{2\ell + d - 1}{\ell} \binom{\ell + d - 2}{\ell-1}
			\quad \ell \geq 1
			\ .
		\end{aligned}
	\end{equation}
	Moreover, the random coefficients $\{a_{\ell m}\}$ are such that
	\begin{equation}\label{alm}
		\E[a_{\ell m} a_{\ell'm'}] = C_\ell \delta_{\ell,\ell'} \delta_{m,m'}, \qquad \ell \in \N,\  m =1, \dots, n_{\ell,d}
	\end{equation}
	where $\{C_{\ell}\}_{\ell\in \N}$ is the angular power spectrum of $T$.
	Without loss of generality, we take the expected value of the field $T$ to be zero; for all $x,y\in\S^d$, using the well-known Schoenberg theorem, the covariance is given by
	\begin{equation}\label{covariance}
		\E[T(x) T(y)] = \sum_{\ell=0}^{\infty } C_\ell \frac{ n_{\ell,d}}{\omega_{d}}G_{\ell,d}(\langle x, y  \rangle )
	\end{equation}
	where $\omega_d$ is the surface area of $\S^d$ and $(G_{\ell, d})_{\ell\in \N}$ is the sequence of the so-called  normalized Gegenbauer polynomials. This is the unique sequence of polynomials such that for each $\ell\in\N$ it holds $\deg G_{\ell,d} = \ell$ and $G_{\ell, d} (1) = 1$, and for any $\ell\neq\ell'\in\N$
	\begin{equation}\label{orto}\int_{-1}^1 G_{\ell,d}(x) G_{\ell',d}(x) (1-x^2)^{d/2-1} \di x = 0\,.%
	\end{equation}
	It is well-known (see for instance~\cite{Szegoe}) that\begin{equation}\label{geg_der}\frac{\mathrm d} {\mathrm dx} G_{\ell,d}(x)= \frac{ \ell ( \ell +d -1 ) }{d}G_{\ell-1, d+2}(x) \ .
	\end{equation}

	\subsection{Random neural networks}
	Let us now recall briefly what we mean by a random neural network. For any given integer $L, n_0, n_1 \dots, n_{L+1}  $,  greater than 1 integers and for any given function $\sigma:\, \R\to \R$ such that $\Gamma_\sigma= \E[\sigma(Z)^2] <\infty$ for $Z\sim \mathcal N(0,1)$, we denote by $(T_s)$ the sequence of random fields given by
	\begin{equation}\label{random_field} T_s (x) = \begin{cases}
			W^{(0)} x + b^{(1)}, & s =0\\
			W^{(s)} \sigma(T_{s-1}(x)) + b^{(s+1)}, & s = 1, \dots, L
		\end{cases}, \qquad x \in \S^{n_0}
	\end{equation}
	where $b^{(s)}  \in \R^{n_{s}}$ and $W^{(s)} \in \R^{n_{s+1} \times n_s}$ are random vectors and matrices with independent components such that $b^{(s)}_i \sim N(0, \Gamma_b )$, $W_{ij}^{(0)} \sim \mathcal N(0, \Gamma_{W_0})$ and $W_{ij}^{(s)} \sim \mathcal N\tonde{0 , \frac{ \Gamma_W}{n_s}}$ for some positive constants $\Gamma_b,\Gamma_{W_0},\Gamma_W$. To simplify the notation, in the sequel we shall take $n_0 = d$.

	It is well established (see \cite{neal2012bayesian, 10.1214/23-AAP1933,cammarota_marinucci_salvi_vigogna_2023,balas24,favaro2023quantitative} and the references therein) that for all $s>1$, as $n_1, \dots, n_{s} \to \infty$, the random field $T_s$ converges weakly to a Gaussian vector field with $n_{s+1}$ i.i.d. centered components $(T_{i;s}^\star)_{i=1,\dots, n_{s+1}}$. We denote the corresponding limiting covariance kernel by $K_s$.  More precisely, assuming the standard calibration condition $\Gamma_b +\Gamma_{W_0} =1$ and taking $\Gamma_W = (1- \Gamma_b)\Gamma_\sigma^{-1}$ (so that each layer has unit variance), we have that $K_s(x,y)$  only depends on the angle between $x$ and $y$. In particular,
	\begin{equation*}
		\begin{aligned}
			&K_s(x,y)=\kappa_s(\langle x , y \rangle)
		\end{aligned}
	\end{equation*}
	with
	\begin{equation}\label{cov_ric} \kappa_s(u) = \kappa_1 \circ \dotsb \circ \kappa_1(u)
	\end{equation}
	composed $s-1$ times. Furthermore  $\kappa_1(1)=\kappa_s(1)=1$ (see \Cref{lem::kernel} below for more details).

	\section{Main results}\label{sec::main} For any $L$, let $T_L:\, \S^d \to \R$ denote one of the components of the limiting Gaussian process $(T_{i;L}^\star)_{i=1,\dotsc, n_{L+1}}$.By the Schoenberg theorem, the covariance kernel of $T_L$ can be expressed as
	\begin{equation} \label{cov_gege}
		K_L(x,y)
		= \kappa_L(\langle x , y\rangle)
		= \sum_{\ell =0}^\infty D_{\ell;\kappa_L} G_{\ell,d}(\langle x, y\rangle)
	\end{equation}
	where $D_{\ell;\kappa_L}$ are non-negative coefficients and $G_{\ell,d}$ are the normalized Gegenbauer polynomials. As $G_{\ell,d}(1)=1$, we can associate to each $T_L$ an integer-valued random variable $X_L$ with the following probability mass function:
	\begin{equation}\label{associated_rv}
		\P(X_L= \ell) = D_{\ell;\kappa_L} \ , \ \ell\in \N .
	\end{equation}
	We call \eqref{associated_rv}
	the \emph{neural network spectral law}. The main idea of our paper is that the asymptotic behavior of this probability distribution when $L$ increases provides insights on important features of the corresponding neural architecture. In particular, our first main result characterizes neural networks by studying the moments of their spectral law.

	\begin{theorem}\label{main_theorem_1}
		Let $\kappa:[-1,1]\to\R$ given by $\E[T_1(x)T_1(y)] = \kappa(\langle x, y \rangle)$. We assume that $\kappa\in C^1$.

		\smallskip

		\begin{itemize}[leftmargin=*]
			\item {\bf Low-disorder.}  If $\kappa'(1)<1$, then:
			\begin{enumerate}[label=\textnormal{(\alph*)}]

				\item If $\kappa$ is infinitely differentiable  in a neighborhood of $1$, then $X_{L}$ has all finite moments. These moments decay exponentially to zero as $L$ diverges, and in particular
				\begin{equation} \label{case1a}\E[ X_L^{n}]=O\tonde{ \kappa'(1)^{L}} \quad as \ L \rightarrow \infty \,. \end{equation}
				\item If $\kappa$ is differentiable only $r$ times in a neighborhood of $1$, then  $X_L$ has only $2r$ finite moments that decay as in~\eqref{case1a}.
			\end{enumerate}
			\item {\bf Sparse.} If $\kappa'(1) = 1$, then
			\begin{equation}\label{self-regu}
				\lim_{L\to \infty} \E[X_L]= 0 \ , \qquad
				1\leq \E[X_L^2]\leq d\ .
			\end{equation}
			Furthermore:
			\begin{enumerate}[label=\textnormal{(\alph*)}]
				\item If $\kappa$ is infinitely differentiable  in a neighborhood of $1$, then $X_{L}$ has all finite moments. The moments greater than the second grow polynomially in $L$.
				In particular, for $ n \geq 2 $, there exist constants $ c_n > 0 $ not depending on $L$ such that
				\begin{equation}\label{case2a} \E[X_L^{2n}] = L^{n-1} ( c_n +o(1))\quad \text{ as } L \to \infty \ .
				\end{equation}.

				\item If $\kappa$ is only $r$ times differentiable in a neighborhood of $1$, then $X_{L}$ has only $2r$ finite moments that grow as in~\eqref{case2a}.
			\end{enumerate}

			\smallskip

			\item {\bf High-disorder.} If $\kappa'(1)>1$, then

			$$
			\lim_{L\to \infty} \P(X_L = 0 ) < 1 $$
			and hence
			$$		 \liminf_{L\to \infty} \E [X_L] > 0 \ .
			$$
			Furthermore:
			\begin{enumerate}[label=\textnormal{(\alph*)}]
				\item If $\kappa$ is infinitely differentiable  in a neighborhood of $1$, then $X_L$ has all finite moments.
				These moments diverge exponentially as $L$ diverges, and in particular there exist constants $ c_n > 0 $ not depending on $L$ such that
				\begin{align}\label{case3a}
					\E[X_L^{2n}]=\kappa'(1)^{nL}(c_n+o(1))\quad \text{as }L\to\infty \ .
				\end{align}
				\item  If $\kappa$ is only $r$ times differentiable in a neighborhood of $1$, then $X_{L}$ has only $2r$ finite moments that grow as~in \eqref{case3a}.
			\end{enumerate}
		\end{itemize}
	\end{theorem}

	The results of \Cref{main_theorem_1} imply the asymptotics schematized in \Cref{tab:P-lim-L2-lim}.

	\begin{table}[h!]
		\centering
		\caption{ Limits of $X_L$ as $ L \to \infty$ in the low-disorder, sparse and high-disorder regimes. $X_L$ does not have a limit in $L^2$ in the sparse and high-disorder regime. In the high-disorder regime,  $X_L$ does not converge to $0$ in probability (but it could converge to some random variable). }
		\begin{tabular}{l | c c}
			& $\P$-$\lim$ & $L^2$-$\lim$ \\
			\hline
			$\kappa'(1)<1$ & $= 0$ & $= 0$ \\
			$\kappa'(1)=1$ & $= 0$ & $\nexists$ \\
			$\kappa'(1)>1$ & $\neq 0$ & $\nexists$
		\end{tabular}

		\label{tab:P-lim-L2-lim}
	\end{table}

	\begin{remark}
		Some bounds in~\Cref{main_theorem_1} are expressed only for even moments.
		Since the random variables $X_L$ are non-negative, inequalities for the odd moments can be derived using the H\"older inequality, namely
		$
		\E[X_L^{2n+1}] \leq \E[X_L^{2n+2}]^{(2n+1)/(2n+2)}
		$.
	\end{remark}
	Our second main result complements \Cref{main_theorem_1} by characterizing the asymptotics of the  sequence of fields $(T_L)_{L \in \N}$ and their derivatives $(T_L^{(r)})_{L\in\N}$, where
	$ T_L^{(r)}(x) = (-\Delta_{\S^d})^{r/2} T(x)  $.

	\begin{theorem}\label{main_theorem_3} Under the same notation and conditions as in  \Cref{main_theorem_1}, we have the following.
		\begin{itemize}[leftmargin=*]
			\item {\bf Low-disorder.}  If $\kappa'(1)<1$, then for all $x\in \S^d$ we have
			\begin{equation}\label{convL2}
				\lim_{L \to \infty }  T_L(x) - T_L({\rm N}) = 0  \quad \text{ in } L^2(\P) \ ,
			\end{equation}
			where ${\rm N} = ( 0, \cdots, 0,1)$ is the north pole.
			Moreover, if $\kappa$ is $s$-times differentiable in a neighborhood of $1$, then, for every $ r = 1 , \dots , s $ and all $x\in \S^d$,
			$$\Var\tonde{T_L^{(r)}(x)}= o(\kappa'(1)^L) \quad  \text{ as }  L \to \infty $$
			and in particular $T_L^{(r)}(x) \to 0 $ in $L^2(\P)$.
			\item {\bf Sparse.} If $\kappa'(1) = 1$, then~\eqref{convL2} holds. Moreover, if $\kappa$ is $s$-times differentiable in a neighborhood of $1$, then, for every $ r =1, \dots,  s$ there exists a constant $ c_r > 0 $ such that, for all $x\in \S^d$,
			$$ \Var \big(T_L^{(r)}(x)\big) = L^{r-1} (c_r + o(1)) \quad \text{ as } L \to \infty$$
			and in particular $ T_L^{(r)}(x) \not\to 0 $ in $ L^2(\P) $.
			\item {\bf High-disorder.} If $\kappa'(1)>1$, then~\eqref{convL2} does not hold. Moreover, if $\kappa$ is $s$-times differentiable in a neighborhood of $1$, then, for every  $ r =1, \dots,  s$ there exists a constant $ c_r > 0 $ such that, for all $x\in \S^d$,
			$$\Var\big(T_L^{(r)}(x)\big) = \kappa'(1) ^{rL}(c_r + o(1)) \quad \text{ as } L \to \infty $$
			and in particular $T_L^{(r)}(x) \not\to 0 $ in $ L^2(\P) $.
		\end{itemize}
	\end{theorem}

	\begin{remark}\label{rmk:sobolev}
		\Cref{main_theorem_3} provides information on the convergence of random neural networks in different Sobolev norms as the depth grows.
		In the low-disorder case, the random field converges to a random constant in $L^2$ and in all Sobolev spaces of corresponding smoothness.
		In the sparse case, the random field converges to a random constant in $L^2$ sense, but not in Sobolev norms.
		In the high-order case, we lose convergence in $L^2$ and all Sobolev norms.
	\end{remark}

	To provide examples, in \Cref{tab::activation} we classify commonly adopted activation functions into one of the three regimes introduced in \Cref{main_theorem_1}; see~\cite{SimmacoThesis} for the explicit computations of $\kappa'(1)$.

	\begin{table}[!h]
				\caption{{Classification of standard activation functions into the three regimes of \Cref{main_theorem_1} for $\Gamma_b = 0$.
				The characters L, S and H stand for Low-disorder, Sparse and High-disorder, respectively.
				$ a_0 \in (1.09,1.1) $ is the positive root of $a^2 \tanh(a^2)-1$. $\Phi(x)$ is the cumulative distribution function of the standard normal distribution.}}
		\renewcommand{\arraystretch}{1.25}
		\centering
		\begin{tabular}{ll|lll}
			{\bf Activation} & $\pmb{ \sigma(x)}$ & \multicolumn{3}{c}{\bf Regime} \\
			\hline
			GELU & $x\Phi(x)$ & L \\
			ReLU & $x\mathbbm I_{x\ge0}$ & & S \\
			LReLU & $x \mathbbm I_{x\ge0}+ 0.01x  \mathbbm I_{x<0}$ & & S \\
			PReLU & $x \mathbbm I_{x\ge0}+ ax  \mathbbm 1_{x<0}$ & & S \\
			RePU $(p\ge2)$& $x^p\mathbbm I _{x\ge0}$ & & & H \\
			Hyperbolic tangent & $\tanh(x)$ & & & H \\
			Normal cdf & $\Phi(x)$ & & & H \\
			Exponential & $e^{ax}$ & L $|a|<1$ & S $|a|=1$ & H $|a|>1$ \\
			Gaussian & $e^{-ax^2/2}$ &  L $a^2< 1\!\!+\!\!\sqrt{2}$ & S $a^2 =1\!\!+\!\!\sqrt{2}$ & H $a^2>1\!\!+\!\!\sqrt{2} $ \\
			Cosine & $\cos(a x)$ & L $|a|<a_0$ & S $|a| = a_0 $ & H $|a|>a_0$
		\end{tabular}
		\label{tab::activation}
	\end{table}

	For isotropic kernels, it is well-known that the corresponding reproducing kernel Hilbert spaces are equivalent to Sobolev spaces, with associated norm penalized by the angular power spectrum \cite{10.1214/14-AAP1067}. This justifies viewing general RKHS as generalized Sobolev spaces.
	In this interpretation, our results suggest that, while Sobolev norms can be blind to depth \cite{bietti2021deep}, the full angular power spectrum itself captures important information,
	and can thus be used to define a proper notion of depth-dependent complexity.
	This is the goal of the next definition.


	\begin{definition}\label{defi::spectral}\emph{(Spectral effective support and spectral effective dimension)}
		Let $T_L:\,\R^d\to\R$ be a random neural network defined as in \eqref{random_field}, and let $X_L$ be the associated random variable as in \eqref{associated_rv}. The spectral $\alpha$-effective support ($\alpha\in (0,1)$) of $T_L$ is
		$$
		\mathcal C_\alpha = \min\Big\{ n \in \N \; \Big| \; \sum_{\ell=0}^ n \P(X_{L}  =\ell) \geq  1-\alpha \Big\}  \,.
		$$
		The spectral $\alpha$-effective dimension of $T_L$ is
		$$ \mathcal D_\alpha = \sum_{\ell=0}^{C_\alpha} n_{\ell,d}$$
		where $n_{\ell,d}$ is defined in  \eqref{dim_spazio}.
	\end{definition}

	\begin{remark}[Interpretation of spectral support and dimension]
		Heuristically, if a neural network with a given architecture has a spectral $\alpha$-effective support equal to $\mathcal C_{\alpha}$, then $\mathcal C_{\alpha}$ multipoles (equivalently, harmonic frequencies) are sufficient to explain $(1-\alpha)\%$ of its random norm (i.e.~variance), and hence the corresponding random field is ``close'' (in the $L^2$ sense) to a polynomial of degree $\mathcal C_\alpha$. The dimension of the corresponding vector spaces depends also on the dimension of the domain $\mathbb{S}^d$, and it grows in general as $\mathcal D_{\alpha} \simeq \mathcal C_{\alpha}^{d-1}$. Some numerical evidence to support these claims is given below, see~\Cref{tab:comp_ReLu}.
	\end{remark}

	\begin{remark}[Spectral complexity] \label{rmk:spectral_complexity}
		In view of \Cref{defi::spectral}, we can reinterpret \Cref{main_theorem_1} as a classification in three distinct complexity regimes.
		Indeed, using the Markov inequality we can obtain bounds $ \mathcal{C}_\alpha \lesssim ( \kappa'(1)^L / \alpha )^{1/2} $.
		Thus, in the low-disorder case, the complexity decays exponentially fast in $L$.
		In the sparse case, the complexity is bounded from above uniformly over $L$.
		In the high-disorder case, we obtain an upper bound that diverges exponentially in $L$;
		while we do not have a matching lower bound, numerical simulations seem to confirm a fast growth (see \Cref{tab:comp_ReLu}).

		To provide a concrete example in the sparse regime,
		one can see that, for $d=2$, more than $0.98 \geq \mathbb{E}[X_L^2]/10^2$ of the probability mass is (uniformly in $L$) in the first $10$ multipoles (i.e. $ \mathcal{C}_{0.02} \le 10 $),
		which span a vector space of dimension $121$ (i.e. $ \mathcal{D}_{0.02} \leq 121 $). This is only an upper bound, and indeed the simulations in~\Cref{tab:comp_ReLu} show for ReLU that up to $99\%$ of the probability mass is confined in the first two multipoles as $L$ increases. This is rather surprising, keeping in mind that, for large $L$, the neural network could have millions, if not billions, of parameters.
	\end{remark}

	\begin{remark}[Spikes in ReLU networks]\label{spikes}
		A comprehensive inspection of the previous results and definitions highlights some peculiarities of the ReLU activation, which may perhaps explain some of its empirical success.
		What observed in \Cref{rmk:sobolev} can be explained by spikes emerging at higher and higher frequencies as the depth increases.
		Because such spikes have asymptotically negligible support, they have limited impact in $L^2$ norm, but they crucially increase the Sobolev energy, resulting in the field derivatives not converging to $0$.
		Further evidence of this fact is given by \Cref{rmk:spectral_complexity}.
		On the one hand, the spectral effective dimension remains bounded at any depth $L$,
		which suggests that ReLU networks are well approximated in $L^2$ by polynomials of low degree, and therefore are less prone to overfitting.
		On the other hand, the divergence of ReLU moments of order larger than $2$ indicates the existence of sparse components at high frequencies, which explains the approximation capabilities of ReLU networks at large depth. Our findings are consistent with deterministic characterizations of the expressivity of ReLU networks, given for instance in \cite{pmlr-v97-hanin19a,HaninRolnick2019b,DaubechiesHanin2022}, and references therein.
	\end{remark}

	\subsection{Idea of the proofs}
	The proofs of \Cref{main_theorem_1,main_theorem_3} are rather technical and are therefore postponed to \Cref{app:proofs}.
	We convey here the main ideas.
	We start from two alternative representations of the covariance function.
	The first one, given in \eqref{cov_ric}, uses the recursive nature of the covariance in the infinite width limit; the second one, given in \eqref{cov_gege}, is the expansion in Gegenbauer polynomials.
	We then compute higher-order derivatives of these covariances in both representations. For \eqref{cov_gege}, this amounts to exploiting explicit formulae for the derivatives of the Gegenbauer functions, and leads to power series involving the spectral weights times powers of the multipoles (see~\Cref{mat_momenti}); for \eqref{cov_ric}, one uses the Faa' di Bruno formula and gets some finite-difference equation, which can be solved by induction and ad-hoc tricks (see~\Cref{new_ex_kn} and~\Cref{new_k1}). Equating the two sets of expressions, further computation provides explicit upper and lower bounds on the spectral coefficients. Since these coefficients are normalized to represent a discrete probability mass function, their asymptotic behavior immediately gives way to a number of concentration results, which in turn lead to~\Cref{main_theorem_1,main_theorem_3}.

	\section{Numerical evidence}\label{sec::numerical}
	We now present numerical experiments to illustrate and support our theoretical results; the corresponding code
	is publicly available at \url{https://github.com/simmaco99/SpectralComplexity}
	and it is based on the HealPIX package for spherical data analysis \cite{gorski2005healpix}.
	In all the following, we consider the standard initialization without bias, i.e.~$ \Gamma_b = 0 $.

	Our first goal is to visualize the different role played by the network depth in the three scenarios considered above in \Cref{main_theorem_1}. In particular, \Cref{fig:3casi,fig:3casi_cont} show random neural networks generated by a Monte Carlo estimation of the angular spectrum, for different activation functions in the three classes (low-disorder, sparse, high-disorder) and growing $L=1, 20, 40, 60, 80$.
	We focus on the following activation functions.
	\begin{itemize}[leftmargin=*]
		\item (Low-disorder case)
		We consider the Gaussian activation function
		$$ \sigma_1(x) = e^{-x^2/2} \ . $$
		By~\Cref{lemm::gaussian}, the normalized covariance kernel is given by
		$$  \kappa(u) =  \sqrt{\frac{3}{4 -u^2}} \ .$$
		A trivial calculation shows that $\kappa'(1)<1.$
		The simulations reported in the first column of \Cref{fig:3casi} are consistent with our results: starting from $L$ of order $20$ we obtain random fields which are very close to constant over the whole sphere. The analogous pictures for higher values of $L$ are not reported in \Cref{fig:3casi_cont}, because they are basically identical.
		\item (Sparse case) Here we take the ReLU activation function
		$$ \sigma_2(x) = \max(0,x)\ .$$
		It is known (see for instance~\cite{cho2011analysis}) that the normalized associated kernel is given by
		$$ \kappa(u) =\frac{ 1}{\pi } \tonde{ u\big( \pi - \arccos(u)\big) + \sqrt{1  -u^2}}\ .$$
		A simple calculation shows that $\kappa'(1) = 1$. The pictures in the middle column of \Cref{fig:3casi} and in the left column of \Cref{fig:3casi_cont} confirm what is expected: the corresponding random field converges to Gaussian realizations with an angular power spectrum dominated by very few low-frequency multipoles, as made evident by the presence of large scale fluctuations at every depth $L$.
		\item (High-disorder case) Let us consider the activation function
		$$ \sigma_3(x) = \tanh(x)\ .$$
		To the best of our knowledge, the associated kernel $\kappa$ is not known analytically. However, \Cref{der_sigmoide} shows that $ \kappa'(1) > 1 $.
		In view of \Cref{main_theorem_1,main_theorem_3}, we know that the corresponding random fields have angular power spectra that diverge to infinity exponentially as the depth increases. Because of this, we expect more and more ``wiggly'' realizations at larger depths; this is confirmed by the plots on the right columns of~\Cref{fig:3casi,fig:3casi_cont}.
		The comparison between ReLU and $\tanh$ realizations are -- we believe -- extremely illuminating.
	\end{itemize}

	To quantify the fluctuations of the fields,
	we computed average minima and maxima through Monte Carlo realizations; the results are reported in~\Cref{tab::max_min}.
	Moreover, we estimate the numerical values of the angular power spectra under different architectures by a Monte Carlo simulation: as expected, while for ReLU the power spectrum concentrates on the very first multipoles as $L$ increases, in the $\tanh$ case it concentrates on higher and higher frequencies.
	Finally, we provide some numerical evidence for our proposed notions of spectral effective support and dimension. In particular, \Cref{effsupprelu} shows that, for ReLU, $99\%$ and more of the angular power spectrum concentrates on very few low multipoles, in the order of 1 or 2; correspondingly, \Cref{effdimrelu} shows that, in the $L^2$ sense, the random neural network is very well approximated by a polynomial function belonging to a vector space of dimension 9 (for $L=60$) or even $4$ (for $L=80$). On the other hand, in the case of $\tanh$, it takes several hundreds of multipoles to achieve a similar $L^2$ approximation in the order of $99\%$ (see the last column of \Cref{effsupptanh}). Correspondingly, it takes a vector space of dimension $10^6$ or larger for a suitable approximation, showing a non-sparse behavior (see the last column of \Cref{effdimtanh}).

	\begin{figure}[!htbp]
		\centering
		\subfloat[$1$ hidden layer]{
			\includegraphics[width = 0.3\textwidth]{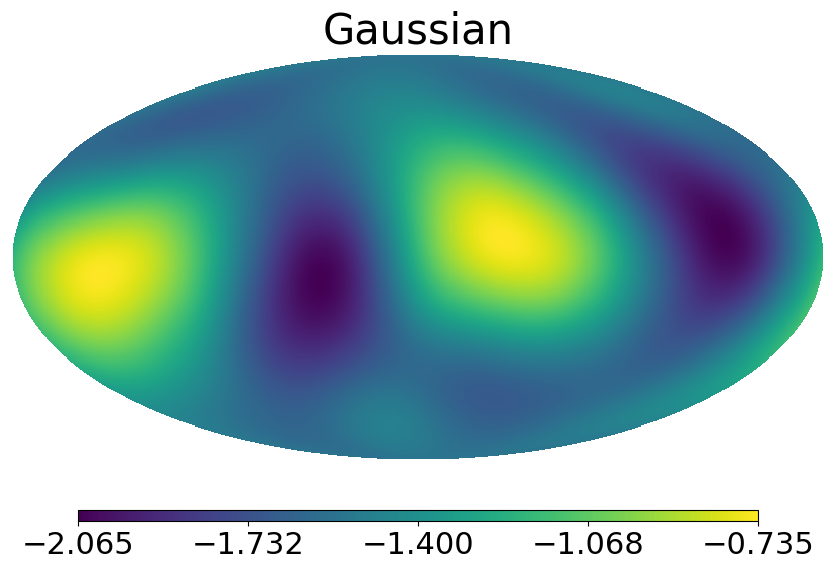}
			\includegraphics[width = 0.3\textwidth]{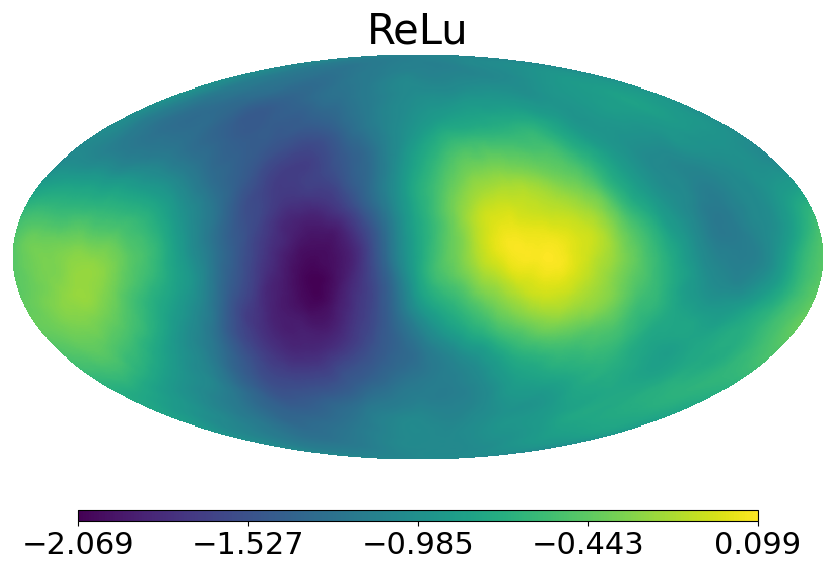}
			\includegraphics[width = 0.3\textwidth]{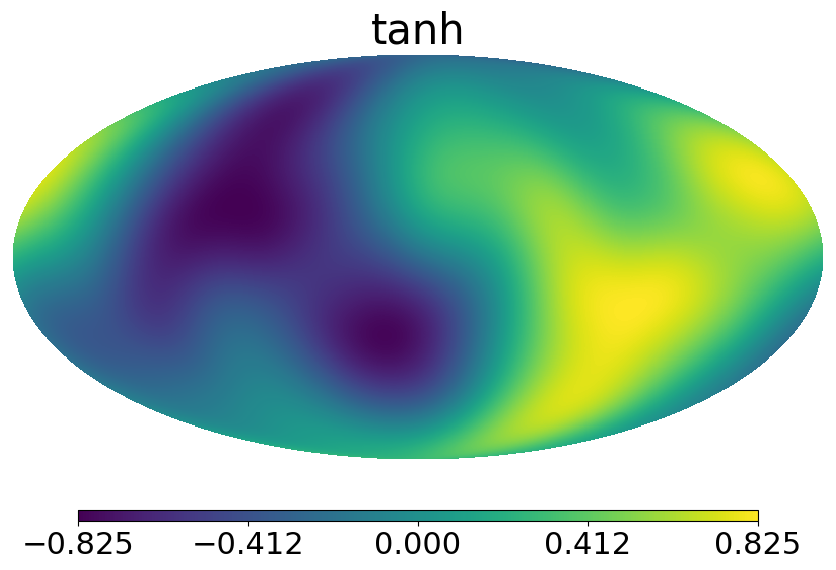}
		}\hfill
		\subfloat[$20$ hidden layers]{
			\includegraphics[width = 0.3\textwidth]{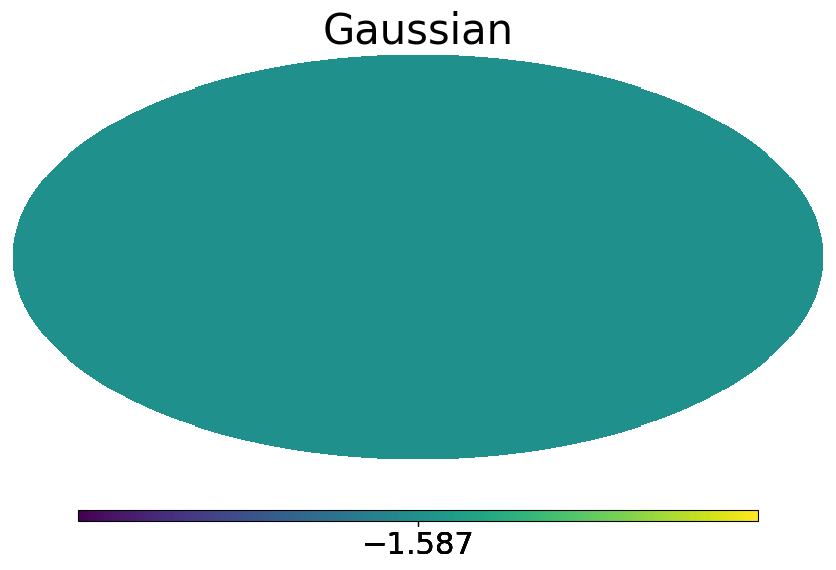}
			\includegraphics[width = 0.3\textwidth]{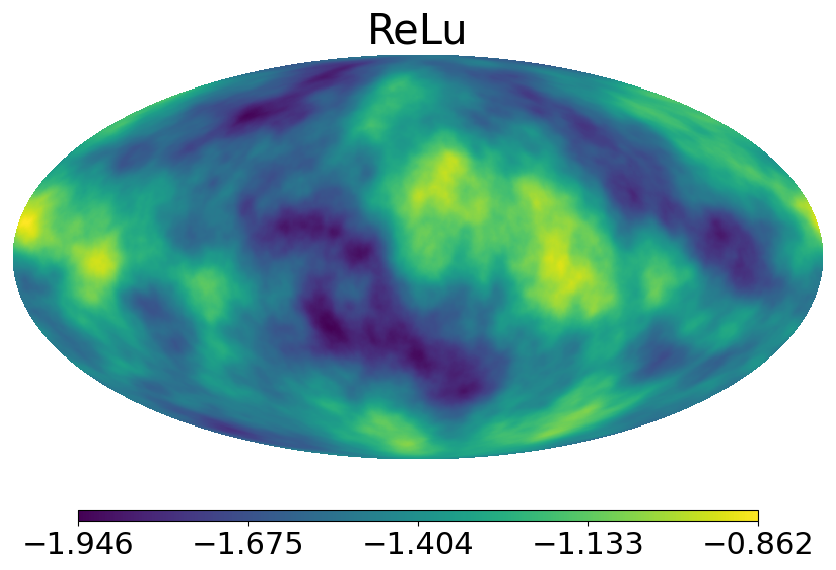}
			\includegraphics[width = 0.3\textwidth]{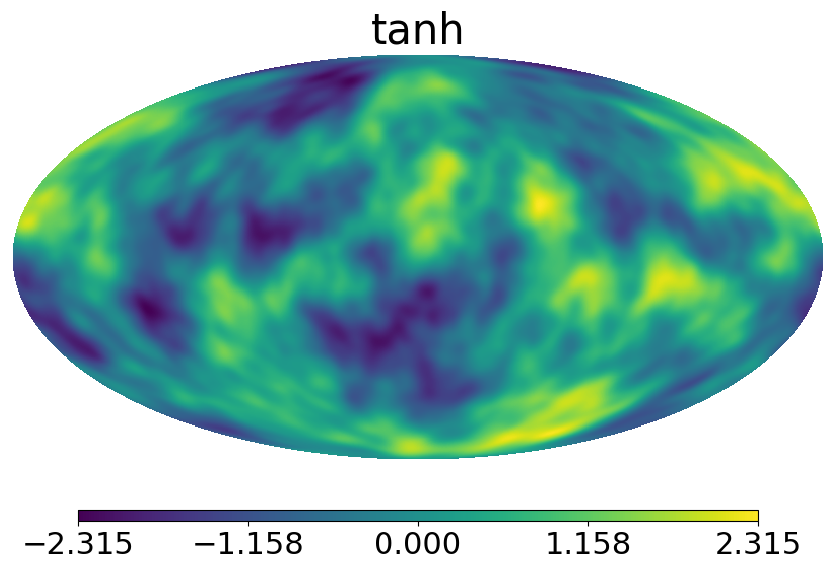}
		}
		\caption{Mollweide projection of a random neural network  $T_L: S^2 \to \mathbb{R}$ with varying depth $L=1,20$ (from top to bottom). The activation functions are $\sigma_1(x) = e^{-x^2/2}$, $ \sigma_2(x) = \max(0,x)$ and  $\sigma_3 = \tanh(x)$ (from left to right). The size of hidden layers is fixed at $1000$ neurons and the resolution of the map is $0.11$ deg. The fields were obtained estimating the angular spectrum by a Monte Carlo estimation (1000 samples) and drawing one realization of the random spherical harmonic coefficients. Note that the color ranges are different from plot to plot.
			See also~\Cref{tab::max_min}, which displays the range of values assumed by the fields; in the plot, the values of the field are approximated to the 3rd decimal digit.
		}
		\label{fig:3casi}
	\end{figure}

	\begin{figure}[!htbp]
		\centering
		\subfloat[$40$ hidden layers]{
			\includegraphics[width = 0.45\textwidth]{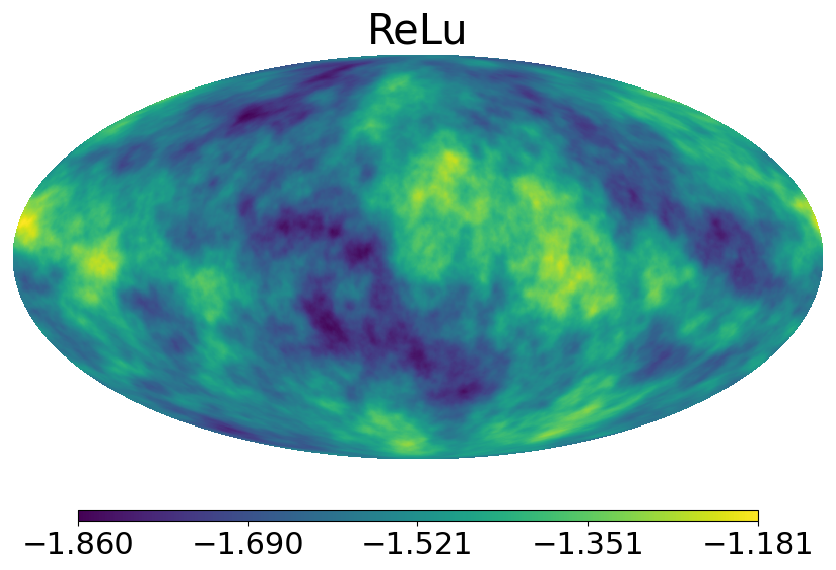}
			\includegraphics[width = 0.45\textwidth]{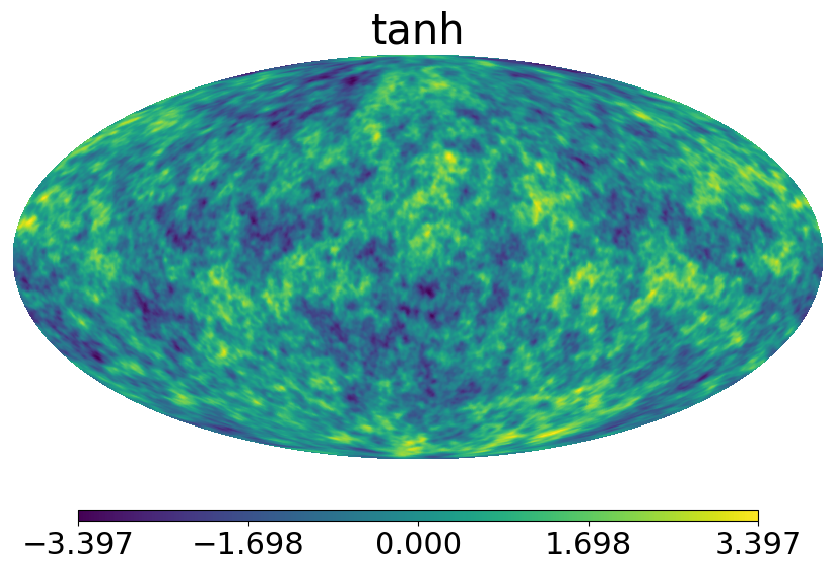}
		}\hfill
		\subfloat[$60$ hidden layers]{
			\includegraphics[width = 0.45\textwidth]{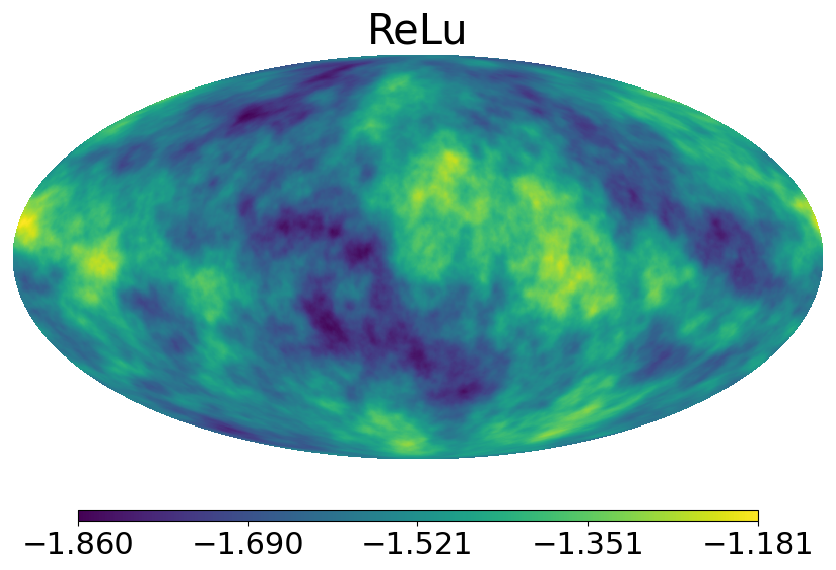}
			\includegraphics[width = 0.45\textwidth]{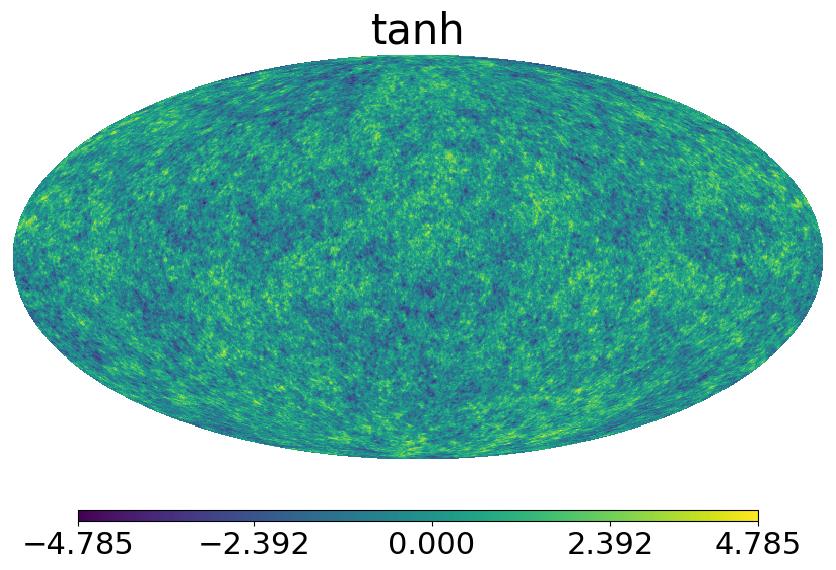}
		}\hfill
		\subfloat[$80$ hidden layers]{
			\includegraphics[width = 0.45\textwidth]{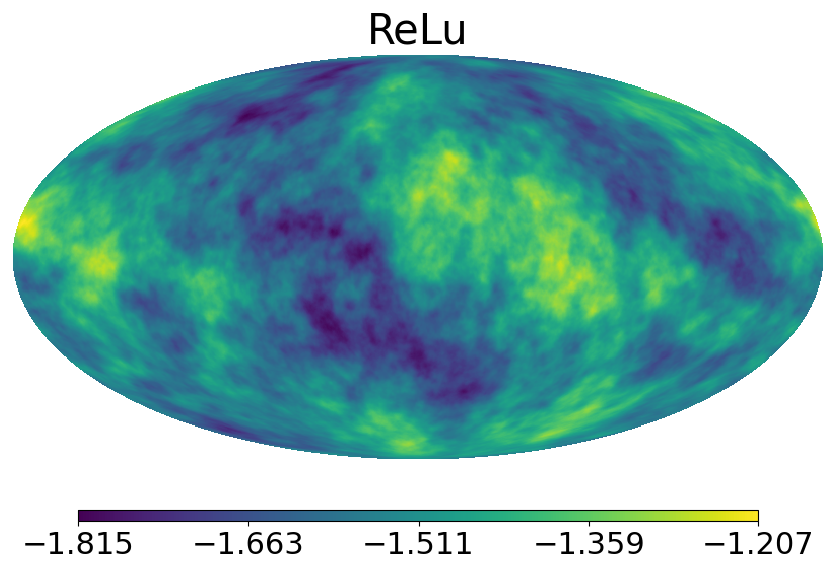}
			\includegraphics[width = 0.45\textwidth]{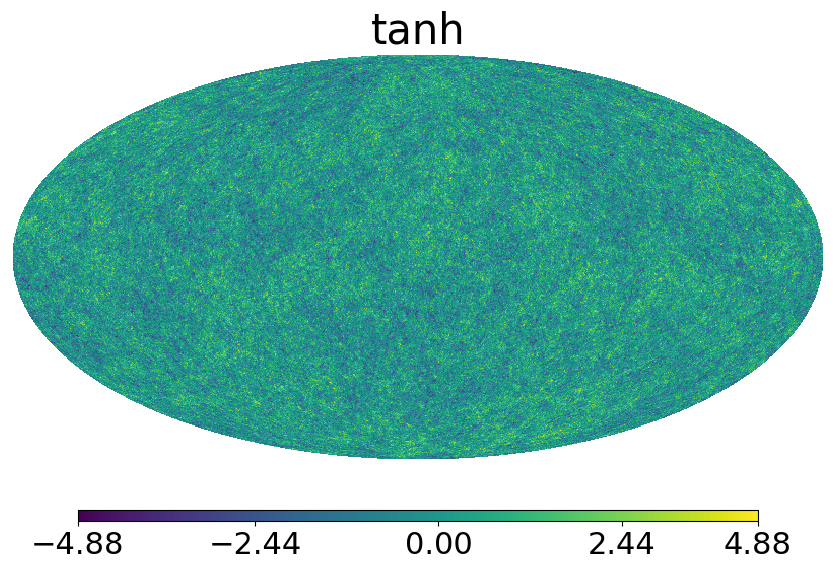}
		}
		\caption{Same as~\Cref{fig:3casi}: ReLU on the left, tanh on the right. The maps corresponding to Gaussian activations are dropped because they are approximately constant on the sphere.}
		\label{fig:3casi_cont}
	\end{figure}

	\begin{table}[!h]
		\caption{The maximum and minimum values of the field with different activation functions; the results were obtained computing the means of $1000$ Monte Carlo realizations.}
		\centering
		\begin{tabular}{c|ll|ll|ll}
			\multirow{2}{*}{\bf depth} & \multicolumn{2}{c|}{\bf Gaussian} & \multicolumn{2}{c|}{\bf ReLU} & \multicolumn{2}{c}{\bf tanh } \\
			& \multicolumn{1}{c}{min} & \multicolumn{1}{c|}{max} & \multicolumn{1}{c}{min} & \multicolumn{1}{c|}{max} & \multicolumn{1}{c}{min} & \multicolumn{1}{c}{max} \\
			1 &  $ -0.55360 $ & $ \phantom{-}0.64094 $ & $ -1.46433 $ & $ 1.53117 $ & $ -1.73473 $ & $ 1.73473 $\\
			20 & $ \phantom{-}0.00013 $ & $ \phantom{-}0.00016 $&$ -0.57855 $ & $ 0.69359 $ & $ -2.92493 $ & $ 2.92493 $\\
			40 & $ \phantom{-}0.07904 $ & $ \phantom{-}0.07904 $& $ -0.42578 $ & $ 0.46043 $ & $ -3.92830 $ & $ 3.92830 $\\
			60 & $ -0.00720 $ & $ -0.00720 $ & $ -0.43158 $ & $ 0.29306 $ & $ -4.65288 $ & $ 4.65288 $\\
			80 & $ \phantom{-}0.00992 $ & $ \phantom{-}0.00993 $ & $ -0.37657 $ & $ 0.26124 $ & $ -4.88463 $ & $ 4.88463 $
		\end{tabular}
		\label{tab::max_min}
	\end{table}
	\newpage
	\begin{table}[!htbp]
		\caption{Spectral $\alpha$-effective supports and dimensions of ReLu and $\tanh$ networks for  $\alpha=0.01,0.005,0.001,0.0005,0.0001$ and $L=1,20,40,60,80$. These values are computed by means of $200$ Monte Carlo realizations, the width of hidden layers is fixed at $500$ neurons and the resolution of the map is $0.11$ deg.}
		\centering
		\subfloat[spectral effective support (ReLU)]{
			\begin{tabularx}{.98\textwidth}{c|YYYYYY}
				\bf depth  & \bf 0.01    & \bf 0.005   & \bf 0.001    & \bf 0.0005   & \bf 0.0001 \\
				\hline
				1     & $2$       & $4$       & $6$       & $8$       & $14$\\
				20    & $4$       & $7$       & $15$      & $20$      & $36$\\
				40    & $3$       & $7$       & $18$      & $25$      & $48$\\
				60    & $2$       & $4$       & $17$      & $25$      & $52$\\
				80    & $1$       & $4$       & $18$      & $27$      & $59$\\
			\end{tabularx}\label{effsupprelu}}
		\vfill
		\subfloat[spectral effective dimension (ReLU)]{
			\begin{tabularx}{.98\textwidth}{c|YYYYYY}
				\bf depth  & \bf 0.01    & \bf 0.005   & \bf 0.001    & \bf 0.0005   & \bf 0.0001 \\
				\hline
				$1$     & $9$       & $25$       & $49$       & $81$        & $225$\\
				$20$    & $25$      & $64$       & $256$      & $441$       & $1369$\\
				$40$    & $16$      & $64$       & $361$      & $676$       & $2401$\\
				$60$    & $9$       & $25$       & $324$      & $676$       & $2809$\\
				$80$    & $4$       & $25$       & $361$      & $784$       & $3600$\\
			\end{tabularx}\label{effdimrelu}}
		\vfill
		\subfloat[spectral effective support ($\tanh$)]{
			\begin{tabularx}{.98\textwidth}{c|YYYYYY}
				\bf depth   & \bf 0.5    & \bf 0.4   & \bf 0.3    & \bf 0.2  & \bf 0.1 & \bf 0.01 \\
				\hline
				$1$     & $1$       & $1$     & $1$      & $1$   & $1$    & $3$\\
				$20$    & $3$       & $3$     & $5$      & $7$   & $9$    & $15$\\
				$40$    & $15$      & $21$    & $29$     & $39$  & $59$    & $135$\\
				$60$    & $77$      & $103$   & $141$    & $195$ & $287$    & $547$\\
				$80$    & $399$     & $541$   & $727$    & $981$ & $1339$    & $>1537$\\
			\end{tabularx}\label{effsupptanh}}
		\vfill
		\subfloat[spectral effective dimension ($\tanh$)]{
			\begin{tabularx}{.98\textwidth}{c|YYYYYY}
				\bf depth   & \bf 0.5    & \bf 0.4   & \bf 0.3    & \bf 0.2  & \bf 0.1 & \bf 0.01 \\
				\hline
				$1$     & $4$       & $4$       & $4$      & $4$      & $4$    & $16$\\
				$20$    & $16$      & $16$      & $36$     & $64$     & $100$    & $256$\\
				$40$    & $256$     & $484$     & $900$    & $1600$   & $3600$    & $18496$\\
				$60$    & $6084$    & $10816$   & $20164$  & $38416$  & $82944$    & $300304$\\
				$80$    & $16000$   & $293764$  & $529984$ & $964324$ & $1795600$    & $>\!\!2.36\!\cdot\!10^6$\\
			\end{tabularx}\label{effdimtanh}
		}
		\label{tab:comp_ReLu}

	\end{table}

	\section{Conclusions and future work}

	In this paper we have introduced a new spectral framework to study the complexity of neural networks with respect to their depth.
	Our results show that increasing the depth can have very different effects depending on the choice of the activation function.
	More precisely,
	we have identified three classes of activations,
	leading to regimes of degeneracy, sparsity, or instability.
	Notably, the ReLU activation falls in the sparse regime,
	suggesting its ability to produce deep hierarchies of features while maintaining a reduced risk of overfitting.

	Our findings open several paths for further research; among these, we mention the following.
	\begin{itemize}

		\item The possibility of exploring the geometry of the random fields associated with a given architecture, for instance, characterizing the expected behavior of their critical points and the Lipschitz-Killing curvatures of their excursion sets (see~\cite{book:73784},~\cite{book:145181}). This topic is already the subject of ongoing research.

		\item The investigation of covariance kernels and the corresponding neural network regime when both width and depth are finite, but tend jointly to infinity. It is not difficult to show that the associated random fields are still isotropic (due to rotational invariance of Gaussian variables), but the behavior of their covariance kernels will be different and presumably depend on the ratio $L/n$, in analogy with what was observed for the asymptotic Gaussianity in~\cite{10.1214/23-AAP1933}.

		\item The investigation of covariance kernels and geometry of excursion sets for more general network architectures, not necessarily feed-forward or fully connected. The most promising alternative seems to be convolutional neural networks, which have very similar associated covariance kernels.

		\item The characterization of the exact distribution of the limiting spectral law associated to different choices of activation functions.

		\item The analysis of the limiting behavior when the dimension $d$ grows and when $d$ and $L$ grow jointly.

	\end{itemize}

	These open problems are currently being explored by the authors.
\newpage
\subsection*{Acknowledgements}
This work was partially supported by the MUR Excellence Department Project MatMod@TOV awarded to the Department of Mathematics, University of Rome Tor Vergata, CUP E83C18000100006. We also acknowledge financial support from the MUR 2022 PRIN project GRAFIA, project code 202284Z9E4, the INdAM group GNAMPA and the PNRR CN1 High Performance Computing, Spoke 3.

	\bibliographystyle{plain}
	\bibliography{reference}
	\newpage
	\appendix 
	\section{Proofs} \label{app:proofs}

	In this appendix we collect all the proofs of the paper. In~\Cref{A1} we derive  analytic relationships between kernel functions and angular power spectra. In~\Cref{A2_new} we study the asymptotic behavior of iterated kernels and their derivatives; we also investigate the asymptotics of spectral moments for the low and high-disorder cases. In~\Cref{A3_new} we study  the intermediate (sparse) case. Finally, in~\Cref{A6}, we give the proofs of our two main theorems. In~\Cref{kernel_comp} we compute the kernel associated to Gaussian activation functions and we bound the derivative of the normalized kernel associated to hyperbolic tangent activation functions.

	\subsection{On the link between kernel derivatives and spectral moments}
	\label{A1}
	Let $\kappa:[-1,1]\to \R$  be a generic continuous covariance kernel corresponding to a unit-variance, isotropic random field on $\mathbb{S}^d$; similarly to \eqref{cov_gege}, by Schoenberg's theorem we have the identity
	\begin{equation}\label{densita1}  \kappa(t) = \sum_{\ell =0}^\infty D_{\ell; \kappa} G_{\ell,d}(t)\ , \ t\in [-1,1]
	\end{equation}
	where $(G_{\ell,d})_{\ell\in\N}$ is the sequence of normalized Gegenbauer polynomials and $(D_{\ell;\kappa})_{\ell\in \N}$  is a sequence of non-negative numbers. Since the field has unit-variance, it holds
	\[
	\kappa(1)= \sum_{\ell=0}^\infty D_{\ell;\kappa}=1 \ .
	\]

	Consistently with \eqref{associated_rv}, we define $X_\kappa$ as the random variable associated to the kernel $\kappa$,
	\begin{equation*}
		\mathbb P(X_\kappa=\ell)  = D_{\ell;\kappa} \ , \ \ell\in \N \ .
	\end{equation*}
	Schoenberg's theorem guarantees that there is a unique variable associated to a given isotropic random field with covariance kernel $\kappa$.

	\begin{proposition}\label{mat_momenti}
		If $\kappa$ is $n$-times differentiable in a neighborhood of $1$, then $X_\kappa$ has $2n$ finite moments. Furthermore, for all $s\leq  n$, these moments satisfy the identity
		\begin{equation}\label{eq:mat_mom}
			\sum_{i=1}^{2s} a_{i;s} \E[X_\kappa^i]  = \frac{(d+2s-2)!!}{(d-2)!!} \kappa^{(s)}(1)
		\end{equation}
		where $\kappa^{(s)}$ is the $s$-th derivative of the kernel $\kappa$ and
		\begin{align*} &a_{i;s+1} = a_{i-2;s} + (d-1) a_{i-1;s} - s(d+s-1) a_{i;s} \qquad & s\geq 0,\ \ & 1\leq i \leq 2(s+1)\ , \\
			&a_{i;s+1} =0 & s\geq 0, \ \   & i \leq 0 \text{ or } i> 2(s+1)
		\end{align*}
		with  the initial condition
		\begin{equation}
			\label{cond_iniziali}a_{i;1} = \begin{cases}
				d-1 & \text{ if } i = 1 \\
				1 & \text{ if } i = 2 \\
				0 & \text{ otherwise}
			\end{cases} \ .
		\end{equation}
		Moreover, $$ a_{2n-1;n} = n(d-1), \qquad
		a_{2n;n} =1 \ .$$
	\end{proposition}
	\begin{proof}
		Using the derivation formula for Gegenbauer polynomials~\eqref{geg_der}, and since the $s$-th derivative of $\kappa$ in a neighborhood of $1$ exists, we can differentiate equation~\eqref{densita1}  $s$ times and interchange the series with the derivative. Hence we obtain
		$$ \kappa^{(s)}(1) = \sum_{\ell=0}^\infty  D_{\ell;\kappa} \tonde{\prod_{j=0}^{s-1} \frac{ (\ell-j)(\ell  + d+j -1)}{d+2j}}\ .$$
		Notice that the first $s-1$ summands are equal to $0$. Denoting by $p_s$ the polynomial
		\begin{equation}
			\label{ps_def}p_s(x)= \prod_{j=0}^{s-1}  (x -j) (x + d + j-1)
		\end{equation}
		we have
		$$ \frac{ (d+2s -2)!!}{(d-2)!!}\kappa^{(s)}(1) = \sum_{\ell=0}^\infty D_{\ell;\kappa} p_s(\ell) = \E\quadre{p_s(X_\kappa)}\ .$$
		It follows that $X_\kappa$ has $2s$ finite moments, as claimed.

		Let us now proceed to establish~\cref{eq:mat_mom}. Since $\deg p_s  =2s$ and $p_s(0)=0$,  there exist $(a_{i;s})_{i=0,\dotsc, 2s}$  such that $a_{0;s} =0$ and
		$$ p_s(x) = \sum_{i=0}^{2s} a_{i;s} x^i \ .$$
		From equation~\eqref{ps_def} it follows that $p_1(x) = x^2 + (d-1)x$ and hence~\eqref{cond_iniziali} is valid. Our idea is to proceed by induction on all other values of $s$. More precisely, from~\eqref{ps_def} and some simple algebraic manipulations we obtain
		$$ p_{s+1}(x) = (x-s)(x+d+s-1) p_s(x)$$
		and in particular
		\begin{align*}
			p_{s+1}(x) =
			&
			\sum_{i=1}^{2s} a_{i;s} x^{i+2} + (d-1)\sum_{i=1}^{2s} a_{i;s} x^{i+1} - s(d+s-1)  \sum_{i=1}^{2s} a_{i;s} x^i \\
			= & \sum_{i=1}^{2(s+1)} \tonde{ a_{i-2;s} + (d-1)a_{i-1;s} - s(d+s-1)a_{i;s}} x^i  = \sum_{i=1}^{2(s+1)} a_{i;s+1} x^i
		\end{align*}
		where to obtain closed-form expressions we have set $a_{i;s}=0$ if $i< 0 $ or $i> 2s$.
		The proof is therefore completed.
	\end{proof}

	\subsection{\texorpdfstring{Moments in low and high-disorder regime ($\kappa'(1)\neq 1$)}{k'(1) not 1}}\label{A2_new}

	Our purpose in this subsection is to derive some analytic expression for the asymptotic behavior of the derivative of iterated covariance kernels  (when the derivative of the covariance at the origin is not $1$).  Combining these results with \Cref{mat_momenti} we will obtain limits of the moments of the spectral law when the depth diverges.
	For the computation of higher-order derivatives of the covariance function, we will exploit the well-known Fa\`a di Bruno formula: if $f$ and $g$ are sufficiently regular functions, then
	$$\frac{\mathrm d^n }{ \mathrm d x^n}  f(g(x)) = \sum_{s=1}^n f^{(s)} (g(x)) B_{n,s}(g'(x), g''(x),\dotsc, g^{(n-s+1)}(x))$$
	where $f^{(i)}$  denotes the  $i$-th derivative of $f$ and $B_{n,s}$ are the incomplete exponential Bell polynomials, given by
	\begin{equation}
		\label{new_Bell_pol}
		B_{n,s}(x_1,\dotsc, x_{n-s+1}) =  n! \sum_{ j \in Q_{n,s}}  \prod_{i=1}^{n-s+1} \frac{x_i^{j_i}}{(i!)^{j_i} j_i!}
	\end{equation}
	where
	$$Q_{n,s} = \left\{ j=(j_1,\dotsc, j_{n-s+1}) \in \N^{n-s+1}\:\left |\;  \sum_{i=1}^{n-s+1} j_i = s , \quad \sum_{i=1}^{n-s+1} i j_i = n  \right\} \right. \ .
	$$
	By the previous definition, it is easy to show that
	\begin{equation}\label{Bell_remark}
		\begin{aligned}B_{n,n}(x_1)= x_1^n\ , \\
			B_{n,1}(x_1,\dotsc, x_n) = x_n \ .
		\end{aligned}
	\end{equation}

	Let $K$ be a continuous angular covariance kernel in dimension $d$.
	We assume that $K$ is an angular kernel, i.e. $K(x,y) = \kappa( \langle x , y \rangle )$.
	We denote by $\kappa_L$ the $L-1$ times composition of $\kappa$, with $\kappa_1=\kappa$.
	In \Cref{lem::kernel} we show that a random neural network with $L$ layers has (angular) kernel of this form.
	We also denote $X_L = X_{\kappa_L}$.
	\begin{proposition}\label{new_ex_kn}

		Let $\kappa$ be  $n$ times differentiable in a neighborhood of $1$. If $\kappa' (1)\neq 1 $, then
		$$ \kappa_L^{(n)}(1) = \sum_{i=1}^n A_{i;n} \kappa' (1)^{Li}$$
		where $A_{1;1} =1$  and $A_{i;n}$ is a constant depending only on the first $n$ derivatives of $\kappa$ at $1$.	Moreover, as $L \to \infty$,
		$$ \kappa_L^{(n)}(1) = \begin{cases} \kappa'(1)^{Ln}\tonde{A_{n;n} + o(1)} & \text { if } \kappa'(1)>1\\
			\kappa'(1)^L \tonde{A_{1;n} + o(1)} & \text{ if } \kappa'(1)<1
		\end{cases} \ . $$
	\end{proposition}
	\begin{proof}
		We prove the claim by induction on $n$. For $n=1$ we have
		$$ \kappa_{L+1}'(x)  = \frac{\di }{\di x} \kappa(\kappa_L(x))  = \kappa'(\kappa_L(x)) \kappa'_L(x)$$
		and so, since $\kappa_L(1)=1$, we have $  \kappa_{L+1}'(1)   = \kappa'(1) \kappa'_L(x) = \kappa'(1)^{L+1} $.
		Let now $n>1$. From Fa\`a di Bruno's formula we obtain
		$$ \kappa_{L+1}^{(n)}(x) = \frac{ \mathrm d^n }{ \mathrm d x^n} \kappa\big(\kappa_L(x)\big) = \sum_{s=1}^n \kappa^{(s)} \big(\kappa_L(x)\big) B_{n,s} \tonde{\kappa_L'(x),\dotsc, \kappa_L^{(n-s+1)}(x)} \ .
		$$
		Hence, denoting by $c_i$ the $i$-th derivative of $\kappa$ at $1$, using $\kappa_L(1)=1$ and $B_{n,1}(x_1,\dotsc, x_n) = x_n$ (cfr.~\eqref{Bell_remark}) we have
		$$ \kappa_{L+1}^{(n)}(1) =  c_1 \kappa_{L}^{(n)}(1) +  \sum_{ s=2}^n c_s B_{n,s} \tonde{\kappa_L'(1),\dotsc, \kappa_L^{(n-s+1)} (1)}$$
		and so we have the following closed formula:
		\begin{equation}\label{closed} \kappa_L^{(n)}(1) = c_1^{L-1} c_n + \sum_{q=1}^{L-1} \sum_{s=2}^n c_s B_{n,s}\tonde{\kappa_q'(1),\dotsc, \kappa_q^{(n-s+1)}(1)} c_1^{L-q-1} \ .
		\end{equation}
		Now consider $s > 1$, so that $n - s + 1 < n$. Using the definition of Bell polynomials (cfr.~\eqref{new_Bell_pol}) and the inductive hypothesis we have
		\begin{align*}
			B_{n,s} \tonde{\kappa_q'(1), \cdots, \kappa_q^{(n-s+1)}(1)} &= n!\sum_{j\in Q_{n,s}}  \prod_{i=1}^{n-s+1} \frac{\tonde{\kappa_q^{(i)}(1)}^{j_i}}{(i!)^{j_i} j_i!	} = n! \sum_{j\in Q_{n,s}}\prod_{i=1}^{n-s+1} \frac{\tonde{ \sum_{ p= 1}^i A_{p;i} c_1^{pq}}^{j_i}}{(i!)^{j_i} j_i!} \\
			& = \sum_{j\in Q_{n,s}} \!\! n! \!\! \prod_{i=1}^{n-s+1} \!\! \frac{1}{(i!)^{j_i}} \!\!  \sum_{k_{1;i}+\dotsb+ k_{i;i}  = j_i} \,
			\prod_{p=1}^i \frac{A_{p;i}^{k_{p;i} }}{k_{p;i}! } c_1^{pqk_{p;i}} \\
			&=
			\sum_{j\in Q_{n,s}} \!\! n! \!\! \prod_{i=1}^{n-s+1} \!\! \frac{1}{(i!)^{j_i}} \!\! \sum_ {s_i= j_i}^{ij_i} c_1^{qs_i} \hspace{-15pt} \sum_{(k_{h;i})_{h=1, \cdots,  i}\atop_{\atop{k_{1;i} +\dotsb + k_{i;i} = j_i\atop{k_{1;i}  +\dotsb + i k_{i;i} =s_i}}}} \hspace{-5pt} \prod_{p=1}^i \frac{A_{p;i}^{k_{p;i} }}{k_{p;i}! }
			=  \sum_{a=s}^n c_1^{qa}  \gimel_{a;n,s}
		\end{align*}
		where in  last equality we have used that for  $j\in Q_{n,s}$ we have $ \sum_{i=1}^{n-s+1} j_i = s$ and $\sum_{i=1}^{n-s+1} i j_i =n$,
		and we have set
		$$ \gimel_{a;n,s} =\sum_{j\in Q_{n,s}}  \sum_{(s_i)_{i=1, \dots, n-s+1}\atop_{s_1 +\dotsb + s_{n-s+1} = a\atop{ j_i \leq s_i \leq ij_i}}} n!  \prod_{i=1}^{n-s+1} \frac{1}{(i!)^{j_i}} \sum_{(k_{h;i})_{h=1, \dots,  i}\atop_{\atop{k_{1;i} +\dotsb + k_{i;i} = j_i\atop{k_{1;i}  +\dotsb + i k_{i;i} =s_i}}}} \prod_{p=1}^i \frac{A_{p;i}^{k_{p;i} }}{k_{p;i}!}\ .$$
		Using the previous equality, one can rewrite~\eqref{closed} as
		\begin{align*} \kappa_L^{(n)}(1) &= c_1^{L-1} c_n + \sum_{q=1}^{L-1}\sum_{s=2}^n c_s \sum_{a=s}^n \gimel_{a;n,s} c_1^{qa + L-q-1}  =
			c_1^{L-1} \tonde{ c_n + \sum_{s=2}^n c_s  \sum_{a=s}^n \gimel_{a;n,s} \sum_{q=1}^{L-1}c_1^{q(a -1)}}\\
			&=
			c_1^{L-1} \tonde{ c_n + \sum_{s=2}^n c_s \sum_{a=s}^n \gimel_{a;n,s} \frac{c_1^{L(a-1)} -c_1^{a-1}}{c_1^{a-1}-1}} \\
			&=c_1^L \tonde{ \frac{c_n}{c_1}  - \sum_{s=2}^n\sum_{a=s}^n \frac{ c_s\gimel_{a;n,s} c_1^a}{c_1^{a-1} -1}} + \sum_{a=2}^n c_1^{La} \tonde{\sum_{s=2}^a \frac{c_s\gimel_{a;n,s}}{c_1^a -c_1}}
		\end{align*}
		whence the claim.
	\end{proof}

	\begin{proposition}\label{new_kno1}
		Let $\kappa$ be infinitely many times differentiable in a neighborhood of $1$.
		If $\kappa'(1)> 1$, then for all $n\in \N\setminus\{ 0\}$ we  have, as $L \to \infty$,
		\begin{align}
			&\E[X_L^{2n}]  =  \kappa'(1)^{nL} \tonde{ C_n+ o(1) },\label{con1}\\
			& \E[X_L^{2n-1}] =
			o(\kappa'(1)^{nL})\label{con2}
		\end{align}
		where
		$$C_n = A_{n;n}\frac{(d+2n -2)!!}{(d-2)!!}  $$
		with $A_{n;n}$  as in~\Cref{new_ex_kn}. Otherwise, if $
		\kappa'(1)<1$, then  for all $n\in \N\setminus\{ 0\}$  we have, as $L\to + \infty,$
		\begin{equation}
			\label{con3}\E[X_L^{n}] = O(\kappa'(1)^L) \ .
		\end{equation}
	\end{proposition}

	\begin{proof} First, we prove the claim for $\kappa'(1)>1$ by induction on $n\geq 1$. To simplify the notation, we put $c_1 = \kappa'(1)$.

		Base case. Let $n=1$,  from \Cref{mat_momenti} and \Cref{new_ex_kn} we have
		\begin{equation}\label{car1}(d-1) \E[X_L] + \E\quadre{X_L^2} = d c_1^L \ .
		\end{equation}
		Since $X_L$ takes only non-negative values we have immediately
		$\E\big[X_L^2\big] \leq d c_1^L\ .$
		On the other hand, by H\"older inequality we obtain
		$$ \E[X_L] \leq \E\big[X_L^2\big]^{ 1/2} \leq \sqrt d c_1^{L/2} \ ,
		$$
		so~\eqref{con2} holds for $n=1$. Moreover,
		$$ \E\big[X_L^2\big] \geq  dc_1^L  - \sqrt d (d-1) c_1^{L/2} = c_1^L ( d+ o(1)), $$
		so~\eqref{con1} holds for $n=1$ (recall that $A_{1,1}=1$).

		Inductive step. We assume that~\eqref{con1} and~\eqref{con2} hold for all $i \leq n$ and we prove that it holds for $n+1$. Let $(a_{i;n+1})_{i=1, \cdots, 2n+2}$ be as in \Cref{mat_momenti}. A trivial computation shows that
		\begin{align*}
			S &=
			\sum_{i=1}^{2(n+1)}a_{i;n+1} \E\quadre{X_L^i} \\
			& =    \sum_{i=1}^n \Big( a_{2i-1;n+1} \E\big[X_L^{2i-1}
			\big] + a_{2i;n+1}\E\big[X_L^{2i}\big]\Big)
			\!+\! (n+1)(d-1) \E\big[X_L^{2n+1}\big] \!+\! \E\big[X_L^{2n+2}\big] .
		\end{align*}
		Since~\eqref{con1} and~\eqref{con2} hold for all $i\leq n$,  each term in the summation is $O\big(c_1^{iL}\big)$, and thus
		$$ S= (n+1)(d-1)\E\big[X_L^{2n+1} \big]+ \E\big[X_L^{2n+2}\big] +O(c_1^{nL}) \ . $$
		Using \Cref{mat_momenti} and  \Cref{new_ex_kn} we have
		\begin{equation}\label{S_def}
			\begin{aligned}
				S = c_1^{(n+1)L} \tonde{ \frac{ (d+2(n+1)-2)!!}{(d-2)!!} A_{n+1,n+1} +o(1)} \ .
			\end{aligned}
		\end{equation}
		Combining the last two statements and using the definition of $C_{n+1}$, we obtain
		\begin{equation}\label{ss1}
			(n+1)(d-1) \E\big[X_L^{2n+1}\big]   +  \E\big[X_L^{2n+2}\big]= c_1^{(n+1)L}(C_{n+1} + o(1) ) \ .
		\end{equation}
		Since $X_L\geq 0$ we have
		$$ \E\big[X_L^{2n+2}\big] \leq c_1^{(n+1)L}(C_n +o(1)) \ .$$
		Using the H\"older inequality we get
		$$
		\E\big[X_L^{2n+1}\big] \leq \E\big[X_L^{2n+2}\big]^{(2n+1)/(2n+2)} \leq c_1^{(n+1/2)L}\tonde{C_{n+1}^{(2n+1)/(2n+2)} +o(1)} = o\Big(c_1^{L(n+1)}\Big)$$
		and thus~\eqref{con2} holds for $n+1$. Moreover, from~\eqref{ss1} we have
		\begin{align*}  \E\big[X_L^{2n+2}\big] =   c_1^{(n+1)L}(C_{n+1} +o(1)) - (n+1)(d-1) \E\big[X_L^{2n+1}\big]\geq c_1^{(n+1)L}\big(C_{n+1}+o(1)\big)
		\end{align*}
		so that~\eqref{con1} holds for $n+1$.

		We move to the case  $\kappa'(1)<1$. We proceed by induction on $n$. From~\eqref{car1}, since $X_L$ takes only non-negative values, we have
		$$ \E[X_L ] \leq \frac{d}{d-1} \kappa'(1)^L $$
		$$  \E[X_L^2 ] \leq d \kappa'(1)^L $$
		hence~\eqref{con3} holds for $n=1$. Let $n\geq 1 $.  Using the same computation as in the proof of  \Cref{new_kno1} we obtain
		$$
		(n+1) (d-1) \E[X_L^{2n+1}] + \E[X_L^{2n+2}] = O(c_1^L) \ .$$
		Since $X_L\geq 0$,~\eqref{con3} holds for $n+1$.
	\end{proof}

	\subsection{\texorpdfstring{Moments in sparse regime ($\kappa'(1)= 1 $)}{k'(1) equal to  1}}\label{A3_new}

	In this section, we will extend the results on the derivatives of iterated covariance  kernels and the asymptotics of moments to the case of covariance with a derivative at the origin equal to 1. To obtain these results, we will use a classic result known as Faulhaber's formula. For any $p$ and $n$  integers we have
	$$ \sum_{k=1}^n k^ p=  \sum_{m=0}^p H_{m,p} n^{p-m+1}$$
	where
	\begin{equation}\label{const_Fau}
		H_{m,p} = \sum_{k=0}^m \sum_{j=0}^k \binom{k}{j} \binom{p+1}{k}\frac{ (-1)^j (j+m)^m}{(k+1)(p+1)} \ .
	\end{equation}

	\begin{proposition}\label{new_k1}
		Let $\kappa$ be  $n$ times differentiable in a neighborhood of $1$. If $\kappa' (1)= 1 $ then
		$$ \kappa_L^{(n)}(1) = \sum_{i=0}^{n-1} F_{i;n} L^{i} $$
		where $F_{1;1} =1$  and $F_{i,n}$ is a constant depending only on the first $n$ derivatve of $\kappa$ at $1$.	Moreover,  as $L \to \infty$,
		$$ \kappa_L^{(n)}(1) =L^{n-1} (F_{n-1;n} + o(1)) \ . $$
	\end{proposition}
	\begin{proof}
		The proof is analogous to the one given in the~\Cref{new_ex_kn}. The base case follows from the same computation as in~\Cref{new_ex_kn}. Now, let $n>1$; using the inductive hypothesis and an argument akin the one in the  proof of~\Cref{new_kno1}, we have
		\begin{align*} &B_{n,s} \big(\kappa_q'(1),\dotsc, \kappa_q^{(n-s+1)}(1)\big) = \sum_{a=s}^{n-s}  \beth_{a;n,s} q^a
		\end{align*}
		where
		$$ \beth_{a;n,s} =\sum_{j\in Q_{n,s}}  \sum_{(s_i)_{i=1,\dotsc, n-s+1}\atop_{s_1 +\dotsb + s_{n-s+1} = a\atop{ 0 \leq s_i \leq ij_i }}} n!  \prod_{i=1}^{n-s+1} \frac{1}{(i!)^{j_i}} \sum_{(k_{h;i})_{h=0,\dotsc,  i-1}\atop_{\atop{k_{0;i} +\dotsb + k_{i-1;i} = j_i\atop{k_{1;i}  +\dotsb + (i -1)k_{i-1;i} =s_i}}}} \prod_{p=0}^{i-1} \frac{A_{p;i}^{k_{p;i} }}{k_{p;i}! } \ .$$

		Using the closed formula~\eqref{closed} and $c_1 =1$ we have
		\begin{align*}
			\kappa_L^{(n)}(1) & =c_n + \sum_{q=1}^{L-1}\sum_{s=2}^n c_s \sum_{a=s}^{n-s} \beth_{a;n,s} q^{a}
			= c_n + \sum_{s=2}^n  \sum_{a=s}^{n-s} c_s\beth_{a;n,s} \sum_{m=0}^{a}H_{m,a}(L-1)^{a-m+1}
		\end{align*}
		where $H_{m;a}$ are the universal constants given in~\eqref{const_Fau}.
		Expanding the binomial $ (L-1)^{a-m+1} = \sum_{i= 0}^{a-m+1} \binom{a-m+1}{i} L^i (-1)^{a-m+1-i} $ and rearranging the terms, we obtain
		\begin{align*}
			\kappa_L^{(n)}(1) = c_n + \sum_{i=0}^{n-1} L^i\sum_{s=2}^{\min(n-i+1, n)}  \sum_{a=i-1}^{n-s}  \sum_{m= 0}^{\min(a, a+1-i)} \binom{a-m+1}{i} c_s\beth_{a;n,s}H_{m,a} (-1)^{a-m+1-i}
		\end{align*}
		which yields the claim.
	\end{proof}

	\begin{proposition}
		Let $\kappa$ be infinitely many times differentiable in a neighborhood of $1$.  If $\kappa'(1)= 1$, then we have, uniformly over $L$,
		\begin{align}\label{m1}
			0 & \leq \E[X_L] \leq 1\ , \\
			\label{m2}1 & \leq \E\big[X_L^2\big]\leq d \ .
		\end{align}
		Moreover, for all $n\geq 2$, as $L \to + \infty$,
		\begin{align}
			\label{con1b}	& \E\big[X_L^{2n}\big] = L^{n-1}(E_n+ o(1)) \ ,  \\
			\label{con2b}	& \E\big[X_L^{2n-1}\big] = O(L^{n-1})
		\end{align}
		where
		$$ E_n  = \frac{ (d+2s-2)!!}{(d-2)!!} F_{n-1,n}$$
		with $F_{n-1,n}$ given in~\Cref{new_k1}.
	\end{proposition}
	\begin{proof} The proof is divided into three parts. In the first part we prove the bounds for first and second moment. In the second part prove~\eqref{con1b} and~\eqref{con2b} when $n =2$.
		Finally, in the third part, we prove \eqref{con1b} and \eqref{con2b} by induction.

		Using  \Cref{mat_momenti} and \Cref{new_k1} we have
		\begin{equation} \label{star}	(d-1)\mathbb{E}[X_L]+\E\big[X_L^2\big]=d \ .
		\end{equation}
		Since $X_L\geq 0$, we have $ \E\big[X_L^2\big] \leq d$.
		Also, we have $\E\big[X_L^2\big] \geq 1$, since otherwise we would have $	\mathbb{E}[X_L] \leq \E\big[X_L^2\big]^{1/2} < 1$, which contradicts \eqref{star}.
		Therefore we get \eqref{m2}.
		Moreover, \eqref{m1} follows from the fact that $X_L$ is non-negative and
		\begin{align*} \E[X_L] = \frac{ d - \E\big[X_L^2\big]}{d-1}\leq 1 \ . \end{align*}
		Next we move to the induction proof for the higher moments.

		Base case. Let $\big(a_{1;2}\big)_{i=1,\dotsc, 4} $ be as in \Cref{mat_momenti}. Using~\Cref{mat_momenti} and \Cref{new_k1} we have
		\begin{equation}\label{prev}a_{1;2}\E[X_L] + a_{2;2}  \E\big[X_L^2\big] + 2(d-1) \E[X_L^3] + \E[X_L^4]  = L(F_{1;2} + o(1))\ .
		\end{equation}
		Using~\eqref{m1} and~\eqref{m2}, we have $a_{1;2}\E[X_L] + a_{2;2}\E\big[X_L^2\big] = O(1)$ as $ L \to + \infty$. So~\eqref{prev} becomes
		$$2(d-1) \E\big[X_L^3] + \E\big[X_L^4\big]  = L(F_{1;2} + o(1))$$
		and we  obtain
		$$ E\big[X_L^4\big]\leq L(F_{1;2} + o(1))\ .$$
		Moreover,
		$$ \E\big[X_L^3\big] \leq \E\big[X_L^4\big]^{3/4} \leq L^{3/4} \Big( F_{1;2}^{3/4} +o(1)\Big) = O(L)\ . $$

		Inductive step. Let $(a_{i;n+1})_{i=1,\dotsc,2n+2}$ be as in~\Cref{mat_momenti} and let
		\begin{align*}S & = \sum_{i=1}^{2(n+1)}a_{i;n+1} \E\big[X_L^i\big]
			\\
			& =   a_{1;n+1} \E[X_L]+ a_{2;n+1}\E\big[X_L^2\big] + \sum_{i=2}^{n} \Big( a_{2i-1;n+1} \E\big[X_L^{2i-1}\big] + a_{2i;n+1}\E\big[X_L^{2i}\big]\Big)\\
			&+ (n+1)(d-1) \E\big[X_L^{2n+1}\big] + \E\big[X_L^{2n+2}\big] \ .
		\end{align*}
		By~\eqref{m1} and~\eqref{m2} we have
		$$ a_{1;n+1} \E[X_L]+ a_{2;n+1}\E[X_L^2]  = O(1) \ .$$
		For all $2\leq i \leq n$, using~\eqref{con1b} and~\eqref{con2b}, we obtain
		$$ a_{2i-1;n+1}\E[X_L^{2i-1}] + a_{2i;n+1}\E[X_L^{2i}]  =O(L^{i-1})\ .$$
		Combining the previous three equalities we have
		$$ S = (n+1)(d-1) \E[X_L^{2n+1}] + \E[X_L^{2n+2}] + O(L^{n-1})\ .$$
		Using~\Cref{mat_momenti},~\Cref{new_k1} and the definition of $E_{n+1}$ we have, as $L\to\infty$,
		$$(n+1)(d-1) \E\big[X_L^{2n+1}\big]   +  \E\big[X_L^{n+2}\big] + O\Big(L^{n-1}\Big) =L^n\big(E_{n+1}+o(1)\big) \ .$$
		Since $X_L\geq 0$, we have
		\begin{align*}
			& \E[X_L^{2n+2}] \leq L^n(F_{n+1} +o(1)) \\
			& \E[X_L^{2n+1}] \leq L^n \tonde{\frac{F_{n+1}}{(n+1)(d-1)} +o(1)} = O(L^{(n+1)-1})
		\end{align*}
		hence~\eqref{con1b} and~\eqref{con2b} hold for $n+1$.
	\end{proof}

	\subsection{Proof of \texorpdfstring{\Cref{main_theorem_1} and \Cref{main_theorem_3}}{main theorems}}\label{A6}

	The next well known result contains a simple yet useful characterization of the covariance kernel for deep neural networks.
	Since we could not find a formal proof of such characterization, we provide one.

	\begin{proposition} \label{ciccio}
		\label{lem::kernel} For all $s \geq 1$, the random field $T_s$ defined in \eqref{random_field} converges weakly in distribution as $n_1,\dotsc, n_s\to \infty$  to a Gaussian process with $n_{s+1}$ i.i.d.~centered components $ ( T^\star_{i;s})_{i=1,\dotsc, n_{s+1}}$. Furthermore, assuming the standard calibration condition $\Gamma_b +\Gamma_{W_0} =1$ and $\Gamma_W = (1- \Gamma_b)\Gamma_\sigma^{-1}$, the  limiting covariance $K_s$ satisfies $K_s(x,y) = \kappa_s(\langle x , y\rangle)$, and for all $s>1$

		$$K_s(x,y) =  \kappa_1 \circ \dots \circ \kappa_1(\langle x, y \rangle)$$
		composed $s-1$ times.
	\end{proposition}

	\begin{proof}[Proof of the compositional structure of $K_s$]
		The limiting covariance function of $T_s$ is given by (see \cite{10.1214/23-AAP1933})
		$$  K_s(x,y) =\begin{cases}
			\Gamma_b + \Gamma_{W_0}\langle x , y\rangle, &s =0\\
			\Gamma_b + \Gamma_W \E_{f\sim \mathcal{GP}(0, K_{s-1})} [\sigma(f(x)) \sigma (f(y))],&s =1,\dotsc , L
		\end{cases}.$$
		Recall that $\E[\sigma(Z)^2]<\infty$ for $Z\sim \mathcal{N}(0,1)$.
		By the completeness of Hermite polynomials, there exists $(J_q(\sigma))_{q\in \N}$ such that
		$$ \sigma(Z) = \sum_{q=0}^\infty \frac{ J_q(\sigma)}{q!} H_q(Z) \ . $$
		In particular, if $(Z_1,Z_2)$ is a centered Gaussian vector with $\E[Z_1^2] = \E[Z_2^2]$, using the well-known diagram formula (see for instance~\cite[Prop. 4.15]{marinucci2011random}) we obtain
		$$ \E[\sigma(Z_1) \sigma(Z_2)] =\sum_{q=0}^\infty \frac{ J_q(\sigma)^2}{q!} \cov(Z_1,Z_2)^q\ .$$
		We now proceed by induction.
		Since $x,y\in \S^d$, then $K_0(x,x) = K_0(y,y)=1$, so
		$$ \E_{f\sim \mathcal{GP}(0, K_{0})}[\sigma(f(x)) \sigma(f(y))]  = \E[\sigma(Z_{1})\sigma(Z_{2})]$$
		where
		$$(Z_{1}, Z_{2}) \sim \mathcal{N}\tonde{0, \begin{pmatrix}
				1 & \Gamma_b + \Gamma_{W_0}\langle x, y \rangle \\
				\Gamma_b + \Gamma_{W_0}\langle x, y \rangle &1
		\end{pmatrix}} . $$
		Thus
		$$K_1(x,y) =  \kappa_1(\langle x, y \rangle) = \Gamma_b + \Gamma_W \sum_{q=0}^\infty \frac{ J_q(\sigma)^2}{q!} (\Gamma_b + \Gamma_{W_0}\langle x, y\rangle)^q\ .$$
		Additionally, using the definition of $\Gamma_\sigma$ we get
		$\kappa_1(1)= K_1(x,x) = \Gamma_b + \Gamma_W\E[\sigma^2(Z_1)] = 1 $.
		We now suppose the claim for $s$ and prove it for $s+1$. Using $K_{s}(x,x) = K_s(y,y)=1$ we obtain
		\begin{align*} K_{s + 1}(x,y) = \Gamma_b + \Gamma_{W} \sum_{q=0}^\infty \frac{ J_q(\sigma)^2}{q!} \tonde{ K_{s}(x,y)}^q = \kappa_1(K_s(x,y))\ .
		\end{align*}
	\end{proof}

	We are now in a position to prove our main results.
	\begin{proof}[Proof of \Cref{main_theorem_1}]
		We combine the previous results in each regime.
		\begin{itemize}
			\item (Low-disorder case) Let $\kappa'(1)<1$. If $\kappa\in C^\infty$, we can use \Cref{new_kno1} to obtain~\eqref{case1a}. Otherwise, if  $\kappa$ is only $r$ times differentiable,  using~\Cref{mat_momenti} we can prove that $X_L$ has only $2r$ finite moments. Following the proof of~\Cref{new_kno1} also when $\kappa$ is only $r$ times differentiable, one can  obtain the asymptotics for these moments.
			\item (Sparse case) Let $\kappa'(1)=1$. The proof for the asymptotics for all the moment greatet than $1$ mirrors that of the previous cases, using~\Cref{new_ex_kn}, insted~\Cref{new_kno1}. Let us now prove that  $\E[X_L]\to 0 $ as  $L\to + \infty$. From~\eqref{cov_gege}, since the Gegenbauer polynomials are orthogonal (cfr.~\eqref{orto}) and since $G_{0,d}(x) = 1$,  we obtain
			\begin{equation}\label{C0_int}\int_{-1}^1 \kappa_L(u) (1-u^2)^{d/2-1} \di u = D_{0;\kappa_L} \int_{-1}^1 (1-u^2)^{d/2-1} \di u  \ .
			\end{equation}
			Using the computation in the proof of \Cref{ciccio}, we have
			$$
			\kappa_1(u) = \Gamma_b + \Gamma_W \sum_{q=0}^\infty \frac{ J_q^2}{q!} u^q  \ . $$
			Since $\kappa\in C^1$, we obtain
			\begin{equation}\label{cont}
				\kappa_1'(u) = \Gamma_W \sum_{q=1}^\infty \frac{ J_q^2}{(q-1)!} u^{q-1}  \ ,
			\end{equation}
			hence $\kappa'(u) <\kappa'(1) = 1$ for all $u \neq \pm1$, thus  $\kappa$ is a contraction on $[-1,1]$. Using the Banach fixed-point theorem, since $1$ is a fixed-point of $\kappa$ we have
			$$ \kappa_L(u) \to 1 \quad \text{for all } u \in (-1,1] \ . $$
			Using the dominated convergence in~\eqref{C0_int} we have
			$$ \lim_{L\to \infty}  D_{0;\kappa_L}=1 \ , $$
			hence $X_L \to 0 $ in probability. Now we prove that $(X_L)$ is uniformly integrable.
			Using the Cauchy--Schwartz and the Markov inequalities, for any $\varepsilon>0$ we have
			$$
			\E\quadre{X_L \mathbbm 1_{[d/\varepsilon, \infty)}(X_L)}
			\leq
			\E\quadre{X_L^2}^{1/2}
			\P(X_L \geq d/\varepsilon)^{1/2}
			\leq
			d^{1/2} \frac{\E[X_L^2]^{1/2}}{d/\varepsilon}
			\leq
			\varepsilon \ .
			$$
			The claim for the first moment follows from the link between convergence in probability of uniformly integrable random variables and convergence in $L^1$.
			\item (High-disorder case) Let $\kappa'(1)>1$. The proof for the  asymptotics for all moments  greater than $1$ mirrors that of the first case; indeed \Cref{new_ex_kn} hold whenever $\kappa'(1) \neq 1$. Let us prove the claim for the first moment.  From~\eqref{cont}, $\kappa'$ is a monotone increasing continuous function on $[0,1]$. We note that  if $\sigma$ is not a polynomial, then $\kappa'(0)<1$, since $\kappa'(0) = \Gamma_WJ_1(\sigma)^2 = 1$ if and only if $J_i(\sigma) =0$ for all $i\geq 2$.  Since $\kappa'$ is an increasing continuous function such that $\kappa'(0)<1$ and $\kappa'(1)>1$, there exists $u_\star$ such that $
			\kappa'(u_\star) =1$. Using the Lagrange theorem it is easy to prove that $\kappa([0,u_\star] )\subseteq [0,u_\star] $ and hence $\kappa$ is a contration in $[0, u_\star]$. From the Banach fixed-point theorem it has a unique fixed-point $u^\star_1\in[0, u^\star]$. In particular, $\kappa_L(x) \to \kappa(u^\star_1) = u^\star_1<1$ as $L\to \infty$, for all $x\in [0,u_\star]$.  Now, using the dominated convergence, we obtain
			$$ \lim_{L\to + \infty} \int_{0}^{u^\star} \kappa_L(x) (1-x^2)^{d/2-1} \di x= u^\star_1 \int_{0}^{u^\star} (1-x^2)^{d/2-1} \di x \ . $$
			From Cauchy-Swartz inequality, we get
			\begin{align*} \lim_{L \to + \infty}  \int_{-1}^1 \kappa_L(x)  (1-x^2)^{d/2-1} \di x  \leq& \int_{-1}^0 (1-x^2)^{d/2-1} \di x + u_1^\star \int_{0}^{ u^\star}(1-x^2)^{d/2-1} \di x \\
				&+\int_{u^\star}^1 (1-x^2)^{d/2-1}\di x
				< \int_{-1}^1  (1-x^2)^{d/2-1} \di x  \ .
			\end{align*}
			Using~\eqref{C0_int} we have
			$$ \lim_{L \to + \infty} D_{0;\kappa_L} <1 \ .$$
			Now,
			$$ \E[X_L] = \sum_{\ell=1}^\infty \ell D_{\ell;\kappa_L} \geq  \sum_{\ell=1}^\infty D_{\ell;\kappa_L} = 1 - D_{0, \kappa_L}$$
			so
			$$ \liminf_{L \to \infty}  \E[X_L] \geq 1 - \lim_{L \to \infty} D_{0, \kappa_L} > 0 $$
			hence the claim holds.
		\end{itemize}

	\end{proof}

	\begin{proof}[Proof of~\Cref{main_theorem_3}]  First we prove the convergence of the covariance.  Let $x\in \S^d$ and $\rm N\in \S^d$ the north pole. From~\eqref{cov_gege} we have
		$$ \kappa_L(x,{\rm N})= \mathrm{Cov}\big(T_L(x), T_L({\rm N})\big) = \sum_{\ell=0}^\infty D_{\ell;L} G_{\ell,d}(\langle  x, {\rm N} \rangle)  = D_{0, L} + \sum_{\ell=1}^\infty D_{\ell;L} G_{\ell,d}(\langle  x, {\rm N} \rangle) \ . $$
		If $\kappa'(1)\leq 1$, using \Cref{main_theorem_1}  we have $\E[X_L] \to 0 $.
		Hence, since $X_L\geq 0 $ we get $D_{0,L} \to 1 $, and since $\sum D_{\ell;L} = 1$ we get $D_{\ell;L} \to 0$ for all $\ell\neq 0$.  On the other hand, if $\kappa'(1)>1$, let $u_1^\star$ as in the proof of~\Cref{main_theorem_1}, and let $x\in \S^d$ such that  $\langle x , {\rm N}\rangle =  u_1^\star$. Then
		$$\kappa_L(\langle x, N\rangle ) \to \kappa(u_1^\star) =  u_1^\star<1 \ .  $$
		Using the definition of $T_L^{(r)}$ and the spherical harmonic decomposition  (cfr.~\eqref{espansione}), since the  spherical  harmonics are eigenfuctions of the Laplace-Beltrami operator (cfr.~\eqref{laplacian}), we obtain
		$$ T_L^{(r)}(x) =\sum_{\ell=0}^\infty \sum_{m=1}^{n_{\ell,d}} \tonde{- \ell(\ell+d-1)}^{r/2} a_{\ell m }Y_{\ell m}(x)
		$$
		and so
		\begin{align*} \Var( T_L^{(r)}(x) ) &= \sum_{\ell, \ell' =0}^\infty \sum_{m=1}^{n_{\ell,d}} \sum_{m'=1}^{n_{\ell',d}}  \tonde{ \ell \ell' (\ell+d-1)(\ell'+d-1)}^{r/2}  Y_{\ell m}(x) Y_{\ell' m'} \E[a_{\ell m} a_{\ell' m'}]\\
			& = \sum_{\ell=0}^\infty  C_{\ell}  \tonde{ \ell  (\ell+d-1)}^{r}  \sum_{m=1}^{n_{\ell,d}}  Y_{\ell m}(x)^2 = \sum_{\ell=0}^{\infty } (\ell(\ell+d-1))^r\frac{C_\ell n_{\ell, d}}{\omega_d} \\
			&= \sum_{\ell=0}^\infty  \ell^r(\ell+ d-1)^r D_{\ell;L}  = \E[X_L^r(X_L+ d-1)^r]
		\end{align*}
		where the second equality holds by the property of the triangular sequence $(a_{\ell,m})$ (cfr.~\eqref{alm}), the third follows from the addiction formula for spherical harmonics~\cite[Theorem 2.9]{atkinson2012spherical}. The claim follows using the asympotics for the moments given in~\Cref{main_theorem_1}.
	\end{proof}

	\subsection{Kernels associated to the Gaussian and the hyperbolic tangent}\label{kernel_comp}
	\begin{lemma}\label{lemm::gaussian} Let $\sigma_1(x) = e^{-a^2x^2/2}$ be the Gaussian activation. Then the normalized associated kernel is given by
		$$ \kappa(u) = \sqrt{ \frac{1 + 2a^2}{(a^2+1)^2 -a^4 u^2}}\ .$$
	\end{lemma}
	\begin{proof}
		It is well-known that
		$$ \sigma_1(x) = \sum_{n=0}^\infty (-1)^n \frac{a^{2n}}{n! (1+a^2)^{n+1/2}2^n} H_{2n}(x) \ .
		$$
		Thus, taking
		$$J_{2q+1}(\sigma)  =0 , \quad J_{2q} = (-1)^q \frac{a^{2q} (2q)!}{q! (1+a^2)^{q+1/2}2^q}$$
		we obtain
		$$ \sigma_1(x) = \sum_{q=0}^\infty \frac{J_{2q}}{(2q)!} H_{2q}(x)\ .$$
		Using once again the diagram formulae, for $Z_1,Z_2$ standard Gaussian random variables with $\E[Z_1Z_2]=u$ we have
		$$\E[\sigma_1(Z_1)\sigma(Z_2)]= \sum_{q=0}^\infty  \frac{J_{2q}^2}{(2q)!} u^{2q} =  \frac{1}{1+a^2} \sum_{q=0}^\infty \frac{(2q)!}{(q!)^2 4^q} \tonde{ \frac{a^2 u}{1+a^2}}^{2q}\ .
		$$
		Using the Taylor expansion of $\arcsin(\alpha x)$
		$$ \arcsin\tonde{\alpha x } = \sum_{n=0}^{\infty} \frac{(2n)!}{4^n (n!)^2 (2n+1)} \tonde{\alpha x}^{2n+1}$$
		we obtain
		$$ \frac{d}{dx} \arcsin\tonde{\alpha x} =  \alpha \sum_{n=0}^\infty \frac{(2n)!}{4^n (n!)^2 } \tonde{\alpha x}^{2n} \ .$$
		Hence, putting $\alpha = a^2/(1+a^2)$ and  after normalizing the variance to $1$, the claim follows.
	\end{proof}

	\begin{lemma}\label{der_sigmoide} Let $ \sigma_3(x) = \tanh(x)$ be the sigmoid activation function.
		Then the derivative of the normalized associated kernel at the origin is greater than $1$.
	\end{lemma}
	\begin{proof} Let $\kappa$ be the normalized associated kernel; by Schoenberg's theorem we have
		\begin{equation}\label{ang_int} \kappa(u) = \sum_{\ell=0}^\infty C_{\ell} \frac{n_{\ell,d}}{\omega_d}, \quad
			C_\ell = A_{d;\ell} \int_{-1}^1 \kappa(t) G_{\ell,d}(t)(1-t^2)^{d/2 -1} \mathrm d t \ ,
		\end{equation}
		for some normalization factor $(A_{d;\ell})_{\ell\in \N}$. By~\eqref{covariance} we have
		$$ \kappa(u)  = \frac{ 1}{\E[\tanh(Z_1)^2]} \E[\tanh(Z_1)\tanh(Z_2)], \quad (Z_1,Z_2) \sim \mathcal{N}\tonde{ 0, \begin{pmatrix}
				1 & u \\ u &1
		\end{pmatrix}} \ . $$
		Since $\tanh$ is odd, $\kappa$ is odd; in view of~\eqref{ang_int}, using $G_{0,d}(t)=1$, we have $C_0=0$.\\
		Now note that $\tanh$ is not a polynomial, hence there must exist infinitely many $\ell>1$ such that $C_\ell \neq  0$; using the derivative formula for Gegenbauer polynomials~\eqref{geg_der} we have
		$$ \kappa'(1) =\sum_{\ell=1}^\infty C_\ell \frac{ n_{\ell,d} \ell(\ell + d -1)}{\omega_{d}}> \sum_{\ell=1}^\infty C_\ell \frac{ n_{\ell,d}}{\omega_d}=1$$
		where the last inequality follows from $\kappa(1) =1$.
	\end{proof}

\end{document}